\documentclass[11pt]{article}

\usepackage[final]{acl}

\usepackage{times}
\usepackage{latexsym}

\usepackage[T1]{fontenc}
\usepackage{subcaption} 

\usepackage[utf8]{inputenc}

\usepackage{microtype}

\usepackage{inconsolata}

\usepackage{graphicx}

\usepackage{listings}

\lstdefinestyle{appendixprompt}{
  basicstyle=\ttfamily\footnotesize,
  columns=fullflexible,
  breaklines=true,
  breakatwhitespace=true,
  keepspaces=true,
  showstringspaces=false,
  frame=single,
  framerule=0.3pt,
  rulecolor=\color{black!25},
  xleftmargin=0.6em,
  xrightmargin=0.2em,
  framexleftmargin=0.3em,
  framexrightmargin=0.3em,
  aboveskip=0.6em,
  belowskip=0.6em,
  escapeinside={(*@}{@*)},
  literate={<deg>}{{\textdegree}}5
}

\usepackage{amsmath,amssymb}
\usepackage{amsthm}
\newtheorem{theorem}{Theorem}[section]

\newtheorem{proposition}[theorem]{Proposition}

\theoremstyle{remark}

\usepackage{algorithm}
\usepackage{algpseudocode}

\usepackage{booktabs} 
\usepackage[table]{xcolor}
\usepackage{multirow}

\usepackage{tikz}
\newcommand*\circled[1]{\tikz[baseline=(char.base)]{
            \node[shape=circle,draw,inner sep=1pt] (char) {#1};}}
\newcommand{\deltaacc}[1]{\textnormal{\scriptsize\,(#1)}}

\title{Revisiting Greedy Decoding for Visual Question Answering:\\A Calibration Perspective}

\author{
 \textbf{Boqi Chen\textsuperscript{1,3}\thanks{These authors contributed equally.}},
 \textbf{Xudong Liu\textsuperscript{2}\footnotemark[1]},
 \textbf{Yunke Ao\textsuperscript{1}\footnotemark[1]},
 \textbf{Jianing Qiu\textsuperscript{3}}
\\
 \textsuperscript{1}ETH Zurich,
 \textsuperscript{2}University of Toronto,
 \textsuperscript{3}MBZUAI
\\
 \small{
   \textbf{Correspondence:} \href{mailto:jianing.qiu@mbzuai.ac.ae}{jianing.qiu@mbzuai.ac.ae}
 }
}

\begin{document}
\maketitle
\begin{abstract}

Stochastic sampling strategies are widely adopted in large language models (LLMs) to balance output coherence and diversity. These heuristics are often inherited in Multimodal LLMs (MLLMs) without task-specific justification. However, we contend that stochastic decoding can be suboptimal for Visual Question Answering (VQA). VQA is a closed-ended task with head-heavy answer distributions where uncertainty is usually epistemic, arising from missing or ambiguous visual evidence rather than plausible continuations. In this work, we provide a theoretical formalization of the relationship between model calibration and predictive accuracy, and derive the sufficient conditions for greedy decoding optimality. Extensive experiments provide empirical evidence for the superiority of greedy decoding over stochastic sampling across multiple benchmarks. Furthermore, we propose Greedy Decoding for Reasoning Models, which outperforms both stochastic sampling and standard greedy decoding in multimodal reasoning scenarios. Overall, our results caution against naively inheriting LLMs decoding heuristics in MLLMs and demonstrate that greedy decoding can be an efficient yet strong default for VQA.
\end{abstract}

\section{Introduction}
\begin{figure}[t]
\centering
\includegraphics[width=0.85\linewidth]{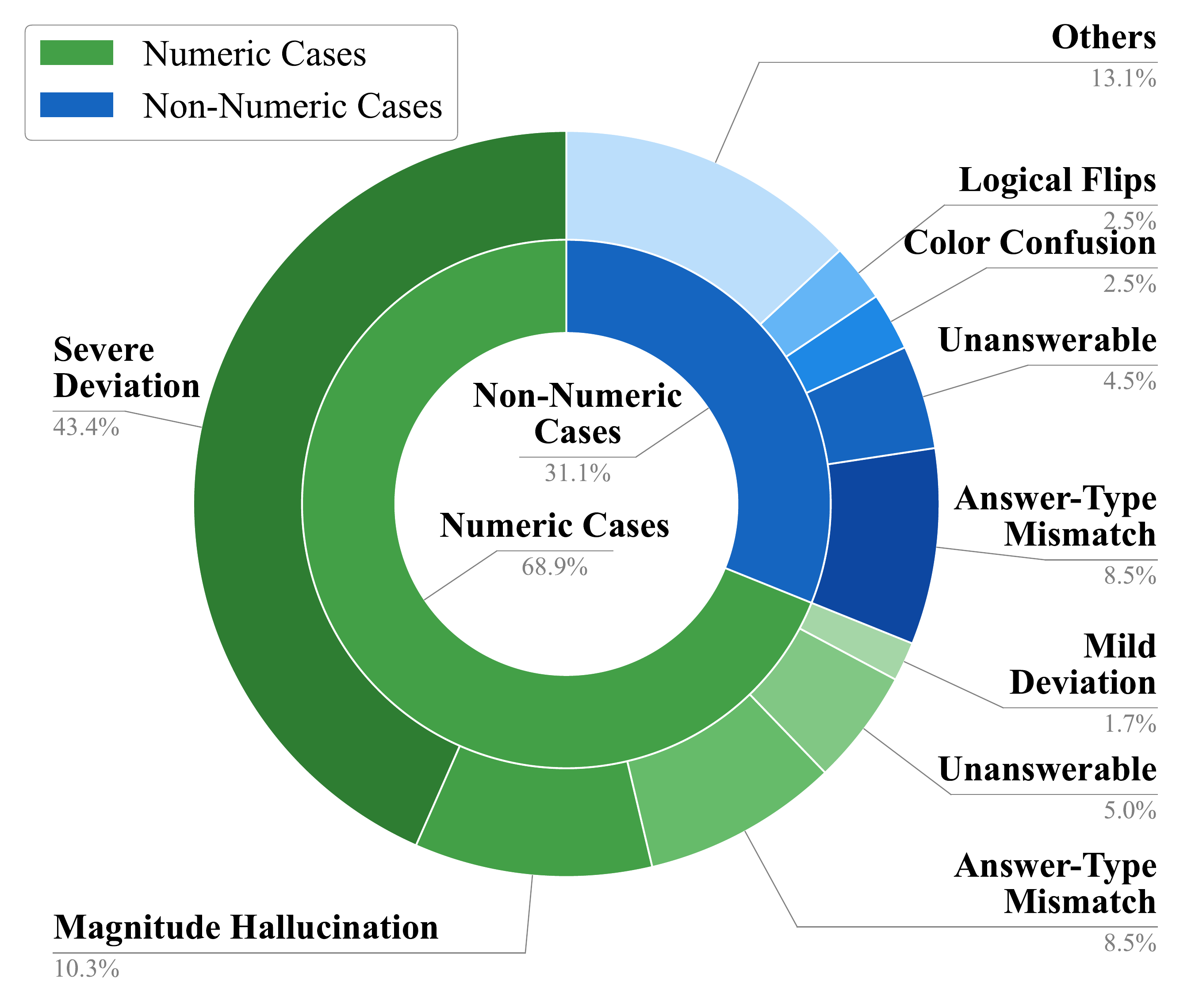}
\caption{\textbf{Failure cases of top-$k$ sampling on ChartQA.} We sample 177 instances in which greedy decoding produces the correct answer, whereas top-$k$ sampling yields an incorrect one.
}
\label{fig:failure_mode}
\end{figure}
Large Language Models (LLMs) have emerged as the standard for various generative tasks due to their ability to produce diverse and high-fidelity outputs. To navigate the inherent trade-off between coherence and diversity, stochastic sampling are widely employed during inference. Common sampling strategies achieve this balance by truncating the probability tail to prevent degeneration while maintaining output variation~\citep{fan2018hierarchical, holtzman2019curious, nguyen2024minp}. Alternatively, entropy-dependent strategies attempt to desmooth the model distribution, treating the output as a mixture of the true distribution and uniform noise to improve generation quality~\citep{hewitt2022desmoothing}.

These stochastic decoding heuristics are often adopted for Multimodal LLMs (MLLMs) without task-specific justification.
However, we contend that Visual Question Answering (VQA) differs from open-ended text generation, rendering sampling during inference suboptimal. 
First, most VQA tasks are closed-ended decision problems that require precise outputs. Therefore, the answer distribution is \emph{head-heavy}, dominated by high-frequency tokens such as numbers, colors, or boolean values~\citep{whitehead2022reliable}. 
For instance, approximately 89\% of answers VQAv2~\cite{antol2015vqa} with real MS-COCO images are single-token, with boolean responses alone constituting 38\% of the distribution. 
This head-heaviness is strong enough that a trivial ``always answer yes'' baseline reaches 29.7\% accuracy. 
Second, VQA relies on visual grounding where uncertainty is primarily \emph{epistemic} rather than aleatoric. 
In open-ended generation, uncertainty is often aleatoric, representing a choice between multiple valid creative trajectories. 
In contrast, uncertainty in VQA typically stems from missing or ambiguous visual evidence~\citep{li2023evaluating, leng2024mitigating, favero2024multi}, under which stochastic sampling risks expanding the candidate set toward low-probability and low-agreement distractors instead of grounded, valid answers. For example, the primary error types introduced by stochastic sampling on ChartQA~\cite{masry2022chartqa} are severe deviations (43.4\%) for numeric questions and answer-type mismatches (8.5\%) for non-numeric ones, whereas semantically equivalent near-misses accounts for only 2.5\% of total failures; see Figure~\ref{fig:failure_mode} and Section~\ref{sec:analytial_res} for a detailed breakdown.


\begin{figure*}[t]
\centering
\includegraphics[width=0.85\linewidth]{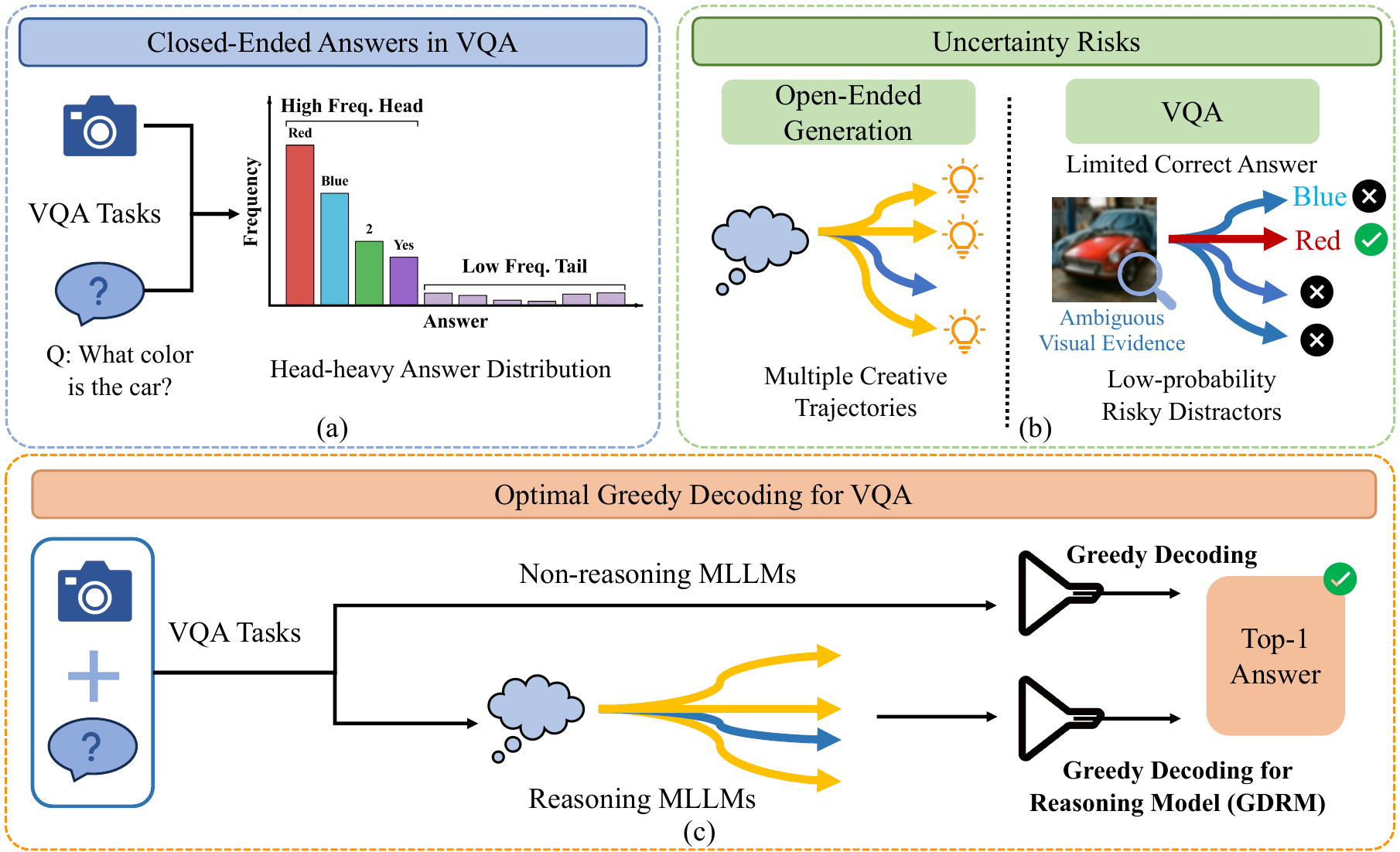}
\caption{\textbf{(a)} VQA Answer distributions are head-heavy (\emph{e.g.}, numbers, boolean values, and colors); \textbf{(b)} VQA uncertainty is primarily epistemic; opening the tail risk introducing low-probability distractors; \textbf{(c)} GDRM anchors final answer tokens to the greedy prediction while preserving internal reasoning capabilities.}
\label{fig:main}
\end{figure*}

In this paper, we revisit decoding strategies from a calibration perspective. 
We first review existing calibration metrics~\cite{Wang2023CalibrationSurvey} and formalize their theoretical relationship to predictive accuracy. 
We then derive sufficient conditions under which greedy decoding is optimal. Extensive experiments show that these conditions are met by MLLMs on VQA tasks, where greedy decoding consistently outperforms stochastic sampling. 
To summarize, our contributions are:
\begin{itemize}
    \item[\circled{1}] We provide a theoretical formalization of the relationship between model calibration metrics and predictive accuracy, and derive sufficient conditions under which greedy decoding is optimal, enabling principled decoding strategy selection via a development set;
    \item[\circled{2}] We conduct extensive evaluations across diverse benchmarks and models and provide empirical evidence for the superiority of greedy decoding over stochastic sampling on VQA tasks;
    \item[\circled{3}] We propose \emph{Greedy Decoding for Reasoning Models} (GDRM) to improve the reasoning model performance which anchors answer tokens to a greedy prediction without compromising the reasoning capability.
\end{itemize}

\section{Related Work}
\paragraph{Decoding in LLMs.}
LLMs employ sampling to balance output diversity and coherence. Temperature scaling modulates the distribution sharpness by rescaling logits but does not restrict the vocabulary, allowing low-probability tokens to emerge at higher temperatures~\cite{holtzman2019curious}. To mitigate degeneration from the unreliable tail, heuristic truncation methods are widely adopted. Top-$k$ sampling~\cite{fan2018hierarchical} constrains the candidate set to the $k$ most probable tokens, while top-$p$ (nucleus) sampling~\cite{holtzman2019curious} retains the smallest set whose cumulative probability mass exceeds a predefined threshold $p$. However, these static cutoffs involve a difficult trade-off where restrictive thresholds may omit plausible continuations and lenient settings risk incoherence. This limitation has motivated more principled truncation strategies that adapt to contextual uncertainty. \citet{hewitt2022desmoothing} propose $\epsilon$-sampling and entropy-dependent $\eta$-sampling to dynamically adjust the candidate set. Most recently, Min-$p$ sampling~\cite{nguyen2024minp} scales the truncation threshold relative to the probability of the top token, $p_{\max}$. This approach enforces stricter pruning when the model is confident and permits greater diversity when the distribution is flat, providing a more robust mechanism for balancing coherence and variation. Recent studies have shown advantages of deterministic decoding~\cite{song2025good, li2025sample}. However, these findings remain heuristic and lack rigorous theoretical analysis.

\paragraph{Decoding in MLLMs.}
MLLMs typically inherit LLM decoding heuristics. However, multimodal generation susceptible to object hallucination, where entities in model outputs lack visual grounding~\cite{li2023evaluating}. Recent progress, therefore, focuses on guided or constrained decoding to improve visual grounding. Visual Contrastive Decoding~\cite{leng2024mitigating} contrasts distributions under real versus perturbed images to downweight tokens driven by language priors. Variants along these lines re-score or filter candidates to emphasize visual faithfulness~\cite{ghosh2024visual, an2025mitigating, su2025activation}. In contrast, our analysis targets the complementary question of deterministic versus stochastic decoding, and thus is applicable on top of these approaches. 

\section{Theoretical Framework}
In this section, we first establish the notation and problem setup (Section~\ref{sec:prob}). We then extend existing calibration metrics to strategy-specific metrics (\emph{i.e.}, $\text{ECE}^\alpha$ and $\text{BS}^\alpha$) for any sampling strategy parameterized by $\alpha$ (Section~\ref{sec:cal}). Finally, we derive sufficient conditions for greedy optimality using these metrics (Section~\ref{sec:greedy}).

\subsection{Problem Setup.}
\label{sec:prob}
For an image--question pair $(I,x)$, let $\mathcal{A}$ be a normalized finite answer set
(\emph{e.g.,} a VQA vocabulary after string canonicalization).
The ground truth \emph{answer-level posterior} is defined as
$p(\cdot|I, x):\mathcal{A}\to[0,1]$ with $\sum_{a\in\mathcal{A}} p(a|I,x)=1$.
Let $q(\cdot| I, x):\mathcal{A}\to[0,1]$ denote the posterior estimated by a MLLM. 

In practice, an answer $a$ is decoded from a sequence of tokens $y_{1:T}$ via a decoding function $D$, such that $a = D(y_{1:T})$, where $y_t$ is the token generated at step $t$ and $T$ is the sequence length. 
The answer probability is computed as the joint product of the token probabilities.

Consider an arbitrary candidate selection strategy parameterized by $\alpha$ (\emph{e.g.,} \emph{top-$k$}, \emph{top-$p$}, or \emph{min-$p$}), where $\alpha$ represents the corresponding hyperparameter value ($k$ for \emph{top-k}, $p$ for \emph{top-p} and $p_{\text{base}}$ for \emph{min-p}). 
For a given input pair $(I, x)$, let $S^\alpha_t(I,x, y_{<t})$ be the set of candidate tokens selected by the strategy at step $t$. 
Under this strategy, tokens are sampled from $S^\alpha_t$ proportionally to the model's original likelihood
\[
y_t\sim \underbrace{\frac{p_\theta\!\left(\cdot\mid y_{<t}, I, x\right)\,\mathbf{1}\{\cdot\in S_t^\alpha\}}{\sum_{y\in S_t^\alpha} p_\theta\!\left(\cdot\mid y_{<t}, I, x\right)}}_{:=p^\alpha_\theta(y\mid I, x, y_{<t})}.
\]
where $\mathbf{1}\{\cdot\}$ is the indicator function and $p_\theta^\alpha$ denotes the resulting sampling distribution.

Ideally, $\alpha$ is chosen to maximize the expectation of producing correct predictions. This objective is defined as the following optimization problem:
\begin{align}
    \max_{\alpha} J(\alpha):=\mathbb{E}_{I,x, a\sim q^\alpha(\cdot|I,x)}[p(a|I,x)] 
\end{align}
where $q^\alpha(a|I,x) = \sum_{D(y_{1:T})=a} p^\alpha_\theta(y_{1:T}|I,x)$ is the probability of selecting answer $a$ under the strategy parameter $\alpha$.
Note that optimization of $J(\alpha)$ is generally intractable in practice due to the unknown ground truth distribution $p$. 
In the following sections, 
%
we identify the sufficient conditions under which the greedy strategy (taking the answer with the highest confidence) is optimal, and connect them to the calibration metrics.
We denote the corresponding answer generated by the greedy strategy as $a^1$.
%


\subsection{Calibration Measurement}
\label{sec:cal}
Several metrics exist for measuring the calibration of deep learning models~\cite{Wang2023CalibrationSurvey}. In this work, we adopt two widely applied measures: Expected Calibration Error (ECE)~\cite{guo2017calibration} and the Brier Score (BS)~\cite{brier1950verification} (definition in Appendix~\ref{sec:cal_metric}). We define strategy-specific calibration metrics $\text{ECE}^\alpha$ and $\text{BS}^\alpha$ as follows.

\paragraph{$\text{ECE}^\alpha$.}
We define $\text{ECE}^\alpha$ directly in terms of the distributions $p$ and $q$ as the expected absolute error over the samples generated by strategy $\alpha$:
\begin{align}
    \text{ECE}^\alpha &:= \mathbb{E}_{I, x, a \sim q^\alpha(\cdot|I, x)} \left[ |q(a|I, x) - p(a|I, x)| \right], \label{def:ece_a}\\ 
    \text{ECE}^1 &:= \mathbb{E}_{I, x} \left[ |q(a^1|I, x) - p(a^1|I, x)| \right],\label{def:ece_1}
\end{align}
where $\text{ECE}^1$ is the calibration error of the greedy strategy (taking the highest-confidence token).

\paragraph{$\text{BS}^\alpha$.}
Similarly, we define the strategy-specific Brier Score as the expected normalized squared difference:
\begin{align}
    \text{BS}^\alpha &:= \mathbb{E}_{I, x, a \sim q^\alpha(\cdot|I, x)} \left[ \frac{|q(a|I, x) - p(a|I, x)|^2}{q(a|I, x)} \right], \label{def:bs_a}\\
    \text{BS}^1 &:= \mathbb{E}_{I, x} \left[ \frac{|q(a^1|I, x) - p(a^1|I, x)|^2}{q(a^1|I, x)} \right],\label{def:bs_1}
\end{align}
where $\text{BS}^1$ is the calibration error of the greedy strategy (always taking the token with the highest confidence). Note that in Eq.~\ref{def:bs_a} we substitute the denominator $p(a|I, x)$ in the original definition of BS (Appendix~\ref{sec:cal_metric}) with $q(a|I, x)$ (see Appendix~\ref{supp-bsk} for details). 

\subsection{Greedy Optimality}
\label{sec:greedy} 
Our primary theoretical contribution is summarized in the following theorem.
\begin{theorem} \label{thr:cal}
Let $\text{ECE}^\alpha$, $\text{ECE}^1$, $\text{BS}^\alpha$, and $\text{BS}^1$ be defined as in~\eqref{def:ece_a},~\eqref{def:ece_1},~\eqref{def:bs_a} and~\eqref{def:bs_1}. Within the family of stochastic sampling strategies parameterized by $\alpha$, greedy decoding is optimal if, for every non-greedy $\alpha$, at least one of the following conditions holds:
\begin{enumerate}
    \item[\circled{1}] $G^\alpha_1 := \mathbb{E}_{I,x, a \sim q^\alpha} [q(a^1|I,x) - q(a|I,x)] \\ - \text{ECE}^1 - \text{ECE}^\alpha \geq 0 \label{cond:1}$
    \item[\circled{2}] $G^\alpha_2 := \mathbb{E}_{I,x, a \sim q^\alpha} \left[ q(a^1|I,x) - \frac{1+q^2(a|I,x)}{2q(a|I,x)} \right] \\- \text{ECE}^1 + \frac{\text{BS}^\alpha}{2} \geq 0 \label{cond:2}$
\end{enumerate}
\end{theorem}

The proof of Theorem~\ref{thr:cal} is in Appendix~\ref{sec:proof}.

Theorem~\ref{thr:cal} suggests that the greedy decoding is optimal under different conditions. Condition~\circled{1} can be satisfied when model assigns high confidence to the top-$1$ answer while maintaining low confidence for others (\emph{i.e.}, large $\mathbb{E}_{I,x,a\sim q^\alpha}[q(a^1|I,x)-q(a|I,x)]$), and is well calibrated across all answers (\emph{i.e.}, low $\text{ECE}^1$ and $\text{ECE}^\alpha$). These requirements align with observed VQA behavior: well-calibrated models exhibit bimodal confidence, with correct predictions concentrated at high confidence and incorrect ones at low confidence, enabling accuracy gains via abstention on uncertain cases~\citep{whitehead2022reliable,khan2024consistency}.
Condition~\circled{2} is easier to satisfy when the top-$1$ prediction is both confident and relatively well-calibrated (\emph{i.e.}, high $q(a^1|I,x)$ and low $\text{ECE}^1$), and lower-ranked ones are substantially less calibrated (\emph{i.e.}, large BS$^\alpha$ for sampling with parameter $\alpha$). 

In practice, although the ground-truth posterior $p$ is not directly observable, $G^\alpha_1$ and $G^\alpha_2$ in Theorem~\ref{thr:cal} can be estimated from a small labeled development set and used as a \emph{pre-hoc diagnostic} for decoding-strategy selection (see Appendix~\ref{sec:appendix_chartqa} for details).

\section{Experiments}
\label{sec:experiments}
\subsection{Benchmark}
\label{subsec:benchmark}
\paragraph{Datasets.}
We use $6$ English benchmarks: (1) MMMU~\cite{yue2024mmmu}, a multidisciplinary multimodal benchmark spanning STEM, humanities, and professional domains; (2) ChartQA~\cite{masry2022chartqa}, which tests reasoning over charts and plots; (3) BLINK~\cite{fu2024blink}, a VQA dataset containing visual commonsense problems; (4) MM-HallBench~\cite{sun2024aligning}, a multimodal hallucination benchmark; (5) MMLU~\cite{mmlu}, a text-only multiple-choice QA benchmark; and (6) CapArena~\cite{cheng2025caparena}, a detailed image captioning benchmark. All results are reported on the official validation splits, which contain 900 (MMMU), 1,920 (ChartQA), 1,901 (BLINK), 96 (MM-HallBench), and 1,531 (MMLU) examples, and 600 images (CapArena).

\paragraph{Models.}
We evaluate three open-source MLLMs: Qwen2.5-VL (3B and 7B)~\cite{bai2025qwen2}, LLaVA-v1.5 (7B)~\cite{Liu2023LLaVA}, and Qwen3-VL-Thinking (4B)~\cite{qwen3vl}. All models are evaluated in the zero-shot setting.

\paragraph{Evaluation metrics.}
Overall accuracy is reported on all datasets except MM-HallBench and CapArena. For multiple-choice questions, accuracy is computed via exact match; for free-form questions, we use normalized soft matching. For stochastic sampling, results are averaged using four different random seeds. Prompt formatting and inference details are in Appendix~\ref{sec:implementation_detail}. We report Answer Score and Hallucination Rate for MM-HallBench, and Average Score and Length for CapArena (detailed definitions are in Appendix~\ref{metric}).

\subsection{Baseline}
\label{subsec:baseline}

\paragraph{Decoding strategies.}
We compare eight decoding strategies: greedy decoding, temperature, top-$k$ \cite{holtzman2019curious}, top-$p$ \cite{holtzman2019curious}, min-$p$ \cite{nguyen2024minp}, $\varepsilon$-Sampling, $\eta$-Sampling \cite{hewitt2022desmoothing}, and Beam Search.

\paragraph{Hyperparameters.}
For each strategy, we sweep a compact grid of commonly used hyperparameters and report the best setting on the validation split. For stochastic sampling methods, we vary temperature $T$ jointly with the strategy-specific hyperparameters; greedy decoding requires no sweep. All sampling baselines in Table 7 use the officially recommended decoding parameters for Qwen3-VL Thinking models, as specified in the Qwen3-VL Technical Report~\cite{qwen3vl}. Hyperparameter details are provided in Appendix~\ref{sec:appendix_hyperparameters}.


\subsection{Main Results}
\begin{table}[t]
\centering
\caption{Accuracy (in \%) of Beam Search on ChartQA benchmark using Qwen2.5-VL (3B).}
\label{tab:qwen3-4b-beam-search}
\small
\begin{tabular}{lccc}
\toprule
\textbf{Method} & $b{=}3$ & $b{=}5$ & $b{=}10$ \\
\midrule
Beam Search & 80.99 & 81.41 & 81.30 \\
\midrule
Greedy & \multicolumn{3}{c}{$\textbf{83.12}$} \\
\bottomrule
\end{tabular}
\end{table}
\begin{table*}[htbp]
\centering
\caption{Accuracy (mean $\pm$ std, in \%) of different decoding and sampling strategies on MMMU, ChartQA, and BLINK using Qwen2.5-VL (3B), averaged over 4 random seeds. Best and second-best are \textbf{bolded} and \underline{underlined}.}
\label{tab:qwen3-3b-bytemp}
\scriptsize
\setlength{\tabcolsep}{3pt}
\resizebox{\textwidth}{!}{%
\begin{tabular}{lccc ccc ccc}
\toprule
& \multicolumn{3}{c}{\textbf{MMMU}} & \multicolumn{3}{c}{\textbf{ChartQA}} & \multicolumn{3}{c}{\textbf{BLINK}} \\
\cmidrule(lr){2-4}\cmidrule(lr){5-7}\cmidrule(lr){8-10}
\textbf{Method} & $\tau{=}0.7$ & $\tau{=}1.0$ & $\tau{=}2.0$ & $\tau{=}0.7$ & $\tau{=}1.0$ & $\tau{=}2.0$ & $\tau{=}0.7$ & $\tau{=}1.0$ & $\tau{=}2.0$ \\
\midrule
Temp' Only             & \underline{43.00 {\scriptsize$\pm$} 0.67} & 41.11 {\scriptsize$\pm$} 0.81 & 33.67 {\scriptsize$\pm$} 0.91 & 79.58 {\scriptsize$\pm$} 0.40 & 76.14 {\scriptsize$\pm$} 0.32 & 48.33 {\scriptsize$\pm$} 1.60 & 32.67 {\scriptsize$\pm$} 1.09 & 28.35 {\scriptsize$\pm$} 0.97 & 23.72 {\scriptsize$\pm$} 1.06 \\
\rowcolor{gray!10}
Top-$k$                & 42.53 {\scriptsize$\pm$} 1.29 & 41.08 {\scriptsize$\pm$} 0.87 & 33.42 {\scriptsize$\pm$} 1.39 & 79.61 {\scriptsize$\pm$} 0.37 & 76.11 {\scriptsize$\pm$} 0.26 & 51.89 {\scriptsize$\pm$} 0.41 & 32.67 {\scriptsize$\pm$} 1.20 & 30.82 {\scriptsize$\pm$} 0.96 & 26.35 {\scriptsize$\pm$} 0.59 \\
Top-$p$ (nucleus)      & 42.69 {\scriptsize$\pm$} 0.43 & \underline{41.30 {\scriptsize$\pm$} 1.14} & 28.20 {\scriptsize$\pm$} 1.83 & \underline{81.13 {\scriptsize$\pm$} 0.43} & 78.93 {\scriptsize$\pm$} 0.47 & 47.71 {\scriptsize$\pm$} 1.52 & 31.61 {\scriptsize$\pm$} 1.28 & \underline{31.67 {\scriptsize$\pm$} 0.74} & 27.14 {\scriptsize$\pm$} 0.56 \\
\rowcolor{gray!10}
Min-$p$                & 42.64 {\scriptsize$\pm$} 1.58 & 40.20 {\scriptsize$\pm$} 0.96 & 28.64 {\scriptsize$\pm$} 0.99 & 79.62 {\scriptsize$\pm$} 0.36 & 76.09 {\scriptsize$\pm$} 0.49 & 45.66 {\scriptsize$\pm$} 2.59 & \underline{33.30 {\scriptsize$\pm$} 0.38} & 30.98 {\scriptsize$\pm$} 0.75 & \underline{29.93 {\scriptsize$\pm$} 1.32} \\
$\varepsilon$-Sampling & 42.45 {\scriptsize$\pm$} 1.90 & 40.50 {\scriptsize$\pm$} 1.34 & \underline{35.03 {\scriptsize$\pm$} 1.45} & 79.91 {\scriptsize$\pm$} 0.32 & \underline{76.53 {\scriptsize$\pm$} 0.51} & \underline{52.46 {\scriptsize$\pm$} 1.63} & 31.09 {\scriptsize$\pm$} 1.04 & 31.19 {\scriptsize$\pm$} 0.65 & 26.25 {\scriptsize$\pm$} 0.42 \\
\rowcolor{gray!10}
$\eta$-Sampling        & \underline{43.00 {\scriptsize$\pm$} 1.76} & 40.44 {\scriptsize$\pm$} 0.62 & 28.09 {\scriptsize$\pm$} 1.72 & 79.66 {\scriptsize$\pm$} 0.34 & 76.46 {\scriptsize$\pm$} 0.43 & 48.16 {\scriptsize$\pm$} 0.56 & 31.09 {\scriptsize$\pm$} 0.91 & 31.04 {\scriptsize$\pm$} 0.27 & 25.51 {\scriptsize$\pm$} 0.95 \\
\midrule
Greedy & \multicolumn{3}{c}{\textbf{45.44}} & \multicolumn{3}{c}{\textbf{83.12}} & \multicolumn{3}{c}{\textbf{35.32}} \\
\bottomrule
\end{tabular}%
}
\end{table*}

Table~\ref{tab:qwen3-3b-bytemp}, Table~\ref{tab:qwen3-7b-bytemp}, and Table~\ref{tab:llava-7b-bytemp} compare greedy decoding with stochastic sampling on MMMU, ChartQA, and BLINK using Qwen2.5-VL (3B and 7B) and LLaVA-v1.5 (7B), respectively. Greedy decoding consistently outperforms stochastic strategies, with the margin over the best stochastic method exceeding the standard deviation in 24 of 27 settings. Notebly, for Qwen2.5-VL, greedy dominates all settings with gap-to-deviation ratios of up to 18.8$\times$.  The only stochastic win occurs for LLaVA-v1.5 (7B) on MMMU at $\tau=1.0$, where Top-$p$ sampling achieves 35.67\% compared with 34.78\% for greedy decoding. LLaVA-v1.5 also exhibits substantially weaker top-$1$ calibration ($\text{ECE}^1$; Figure~\ref{fig:ece-topk}), consistent with our theoretical analysis that poor top-$1$ calibration weakens the sufficient conditions for greedy optimality. Importantly, failure of these sufficient conditions does not itself identify which alternative decoding strategy will be optimal; rather, it indicates that greedy decoding is no longer theoretically certified and that alternative strategies should be compared empirically (see Section~\ref{sec:analytial_res} for detailed analysis). 

Table~\ref{tab:qwen3-4b-beam-search} compares beam search with greedy decoding using Qwen2.5-VL (3B) on ChartQA, where answers are multi-token rather than multiple-choices (\emph{i.e.}, \texttt{A/B/C/D}). Greedy decoding achieves higher accuracy than beam search despite a lower computational budget. Finally, human evaluation of ChartQA outputs by two researchers yields high agreement with string-match labels (Cohen's $\kappa=0.95$), suggesting our results are not driven by string-matching artifacts.

Figure~\ref{fig:acc-temp} shows that stochastic sampling is highly sensitive to temperature across models and datasets. Specifically, performance degrades rapidly as temperature increases, reflecting an inverse relationship between randomness and accuracy that aligns with the head-heavy answer distributions in VQA. Conversely, as temperature decreases, sampling approximates greedy decoding and performance improves. This convergence toward greedy decoding as the optimal limit is precisely the phenomenon our theoretical framework explains. Overall, these results indicate that greedy decoding is an effective and robust decoding strategy for VQA.



\begin{table*}[htbp]
\centering
\caption{Accuracy (mean $\pm$ std, in \%) of different decoding and sampling strategies on MMMU, ChartQA, and BLINK using Qwen2.5-VL (7B), averaged over 4 random seeds. Best and second-best are \textbf{bolded} and \underline{underlined}.}
\label{tab:qwen3-7b-bytemp}
\scriptsize
\setlength{\tabcolsep}{3pt}
\resizebox{\textwidth}{!}{%
\begin{tabular}{lccc ccc ccc}
\toprule
& \multicolumn{3}{c}{\textbf{MMMU}} & \multicolumn{3}{c}{\textbf{ChartQA}} & \multicolumn{3}{c}{\textbf{BLINK}} \\
\cmidrule(lr){2-4}\cmidrule(lr){5-7}\cmidrule(lr){8-10}
\textbf{Method} & $\tau{=}0.7$ & $\tau{=}1.0$ & $\tau{=}2.0$ & $\tau{=}0.7$ & $\tau{=}1.0$ & $\tau{=}2.0$ & $\tau{=}0.7$ & $\tau{=}1.0$ & $\tau{=}2.0$ \\
\midrule
Temp' Only             & 47.78 {\scriptsize$\pm$} 0.74 & 45.67 {\scriptsize$\pm$} 0.88 & 40.22 {\scriptsize$\pm$} 1.02 & 78.96 {\scriptsize$\pm$} 0.36 & 77.97 {\scriptsize$\pm$} 0.41 & 58.38 {\scriptsize$\pm$} 1.28 & 37.56 {\scriptsize$\pm$} 0.83 & 35.61 {\scriptsize$\pm$} 0.79 & 32.25 {\scriptsize$\pm$} 0.94 \\
\rowcolor{gray!10}
Top-$k$                & 47.78 {\scriptsize$\pm$} 0.91 & 46.00 {\scriptsize$\pm$} 0.76 & 38.22 {\scriptsize$\pm$} 1.31 & 78.96 {\scriptsize$\pm$} 0.33 & 77.92 {\scriptsize$\pm$} 0.35 & 61.30 {\scriptsize$\pm$} 0.96 & 37.56 {\scriptsize$\pm$} 0.88 & 35.61 {\scriptsize$\pm$} 0.82 & \underline{35.35 {\scriptsize$\pm$} 0.67} \\
Top-$p$ (nucleus)      & 46.56 {\scriptsize$\pm$} 0.69 & \underline{46.33 {\scriptsize$\pm$} 0.92} & 40.33 {\scriptsize$\pm$} 1.17 & 80.00 {\scriptsize$\pm$} 0.39 & \underline{79.69 {\scriptsize$\pm$} 0.44} & 67.60 {\scriptsize$\pm$} 1.12 & \underline{38.77 {\scriptsize$\pm$} 0.71} & \underline{37.45 {\scriptsize$\pm$} 0.68} & 34.03 {\scriptsize$\pm$} 0.73 \\
\rowcolor{gray!10}
Min-$p$                & 47.33 {\scriptsize$\pm$} 1.12 & 44.11 {\scriptsize$\pm$} 0.95 & \underline{44.33 {\scriptsize$\pm$} 0.86} & 71.41 {\scriptsize$\pm$} 1.84 & 71.56 {\scriptsize$\pm$} 1.76 & \underline{71.51 {\scriptsize$\pm$} 1.69} & 29.40 {\scriptsize$\pm$} 1.21 & 30.30 {\scriptsize$\pm$} 1.08 & 30.25 {\scriptsize$\pm$} 1.02 \\
$\varepsilon$-Sampling & 49.67 {\scriptsize$\pm$} 1.36 & 45.44 {\scriptsize$\pm$} 1.01 & 40.33 {\scriptsize$\pm$} 1.09 & \underline{80.16 {\scriptsize$\pm$} 0.31} & 78.18 {\scriptsize$\pm$} 0.48 & 69.79 {\scriptsize$\pm$} 1.24 & 37.30 {\scriptsize$\pm$} 0.76 & 35.77 {\scriptsize$\pm$} 0.71 & 31.78 {\scriptsize$\pm$} 0.81 \\
\rowcolor{gray!10}
$\eta$-Sampling        & \underline{49.56 {\scriptsize$\pm$} 1.28} & 44.67 {\scriptsize$\pm$} 0.83 & 39.11 {\scriptsize$\pm$} 1.26 & \underline{80.16 {\scriptsize$\pm$} 0.34} & 78.73 {\scriptsize$\pm$} 0.46 & 66.20 {\scriptsize$\pm$} 1.07 & 37.30 {\scriptsize$\pm$} 0.79 & 35.82 {\scriptsize$\pm$} 0.64 & 33.82 {\scriptsize$\pm$} 0.77 \\
\midrule
Greedy & \multicolumn{3}{c}{\textbf{52.56}} & \multicolumn{3}{c}{\textbf{82.13}} & \multicolumn{3}{c}{\textbf{41.56}} \\
\bottomrule
\end{tabular}%
}
\end{table*}

\begin{table*}[t]
\centering
\caption{Accuracy (mean $\pm$ std, in \%) of different decoding and sampling strategies on MMMU, ChartQA, and BLINK using LLaVA-1.5 (7B), averaged over 4 random seeds. Best and second-best are \textbf{bolded} and \underline{underlined}.}
\label{tab:llava-7b-bytemp}
\scriptsize
\setlength{\tabcolsep}{3pt}
\resizebox{\textwidth}{!}{%
\begin{tabular}{lccc ccc ccc}
\toprule
& \multicolumn{3}{c}{\textbf{MMMU}} & \multicolumn{3}{c}{\textbf{ChartQA}} & \multicolumn{3}{c}{\textbf{BLINK}} \\
\cmidrule(lr){2-4}\cmidrule(lr){5-7}\cmidrule(lr){8-10}
\textbf{Method} & $\tau{=}0.7$ & $\tau{=}1.0$ & $\tau{=}2.0$ & $\tau{=}0.7$ & $\tau{=}1.0$ & $\tau{=}2.0$ & $\tau{=}0.7$ & $\tau{=}1.0$ & $\tau{=}2.0$ \\
\midrule
Temp' Only             & 32.44 {\scriptsize$\pm$} 0.84 & 32.89 {\scriptsize$\pm$} 0.91 & 31.44 {\scriptsize$\pm$} 1.02 & 21.93 {\scriptsize$\pm$} 0.73 & \underline{19.69 {\scriptsize$\pm$} 0.66} & 8.80 {\scriptsize$\pm$} 1.18 & 38.35 {\scriptsize$\pm$} 0.72 & 38.77 {\scriptsize$\pm$} 0.69 & 35.61 {\scriptsize$\pm$} 0.88 \\
\rowcolor{gray!10}
Top-$k$                & 31.33 {\scriptsize$\pm$} 0.97 & 29.22 {\scriptsize$\pm$} 1.08 & 29.11 {\scriptsize$\pm$} 1.14 & \underline{22.03 {\scriptsize$\pm$} 0.61} & 18.33 {\scriptsize$\pm$} 0.79 & 9.69 {\scriptsize$\pm$} 1.05 & 38.35 {\scriptsize$\pm$} 0.67 & \underline{39.77 {\scriptsize$\pm$} 0.64} & 35.66 {\scriptsize$\pm$} 0.81 \\
Top-$p$ (nucleus)      & 32.33 {\scriptsize$\pm$} 0.76 & \textbf{35.67 {\scriptsize$\pm$} 0.88} & 31.11 {\scriptsize$\pm$} 0.95 & \underline{22.03 {\scriptsize$\pm$} 0.58} & \underline{19.69 {\scriptsize$\pm$} 0.63} & 9.74 {\scriptsize$\pm$} 0.96 & \underline{38.82 {\scriptsize$\pm$} 0.61} & 39.45 {\scriptsize$\pm$} 0.62 & 36.14 {\scriptsize$\pm$} 0.74 \\
\rowcolor{gray!10}
Min-$p$                & 32.44 {\scriptsize$\pm$} 1.03 & 32.00 {\scriptsize$\pm$} 0.92 & 29.67 {\scriptsize$\pm$} 1.21 & 21.51 {\scriptsize$\pm$} 0.81 & 18.18 {\scriptsize$\pm$} 0.72 & 2.97 {\scriptsize$\pm$} 1.43 & 38.35 {\scriptsize$\pm$} 0.75 & 38.77 {\scriptsize$\pm$} 0.71 & 33.77 {\scriptsize$\pm$} 0.97 \\
$\varepsilon$-Sampling & 32.56 {\scriptsize$\pm$} 0.89 & 31.22 {\scriptsize$\pm$} 0.96 & 28.22 {\scriptsize$\pm$} 1.17 & 21.51 {\scriptsize$\pm$} 0.76 & 18.18 {\scriptsize$\pm$} 0.68 & 11.77 {\scriptsize$\pm$} 1.09 & 38.24 {\scriptsize$\pm$} 0.69 & 38.56 {\scriptsize$\pm$} 0.66 & 34.24 {\scriptsize$\pm$} 0.86 \\
\rowcolor{gray!10}
$\eta$-Sampling        & \underline{34.78 {\scriptsize$\pm$} 0.93} & 32.11 {\scriptsize$\pm$} 0.85 & 28.44 {\scriptsize$\pm$} 1.10 & 21.09 {\scriptsize$\pm$} 0.72 & 19.01 {\scriptsize$\pm$} 0.65 & 2.76 {\scriptsize$\pm$} 1.51 & 37.08 {\scriptsize$\pm$} 0.83 & 37.08 {\scriptsize$\pm$} 0.79 & \underline{37.08 {\scriptsize$\pm$} 0.71} \\
\midrule
Greedy & \multicolumn{3}{c}{\underline{34.78}} & \multicolumn{3}{c}{\textbf{23.54}} & \multicolumn{3}{c}{\textbf{39.93}} \\
\bottomrule
\end{tabular}%
}
\end{table*}

\begin{figure}[t]
\centering
\begin{subfigure}[b]{0.67\linewidth}
  \centering
  \includegraphics[width=\linewidth]{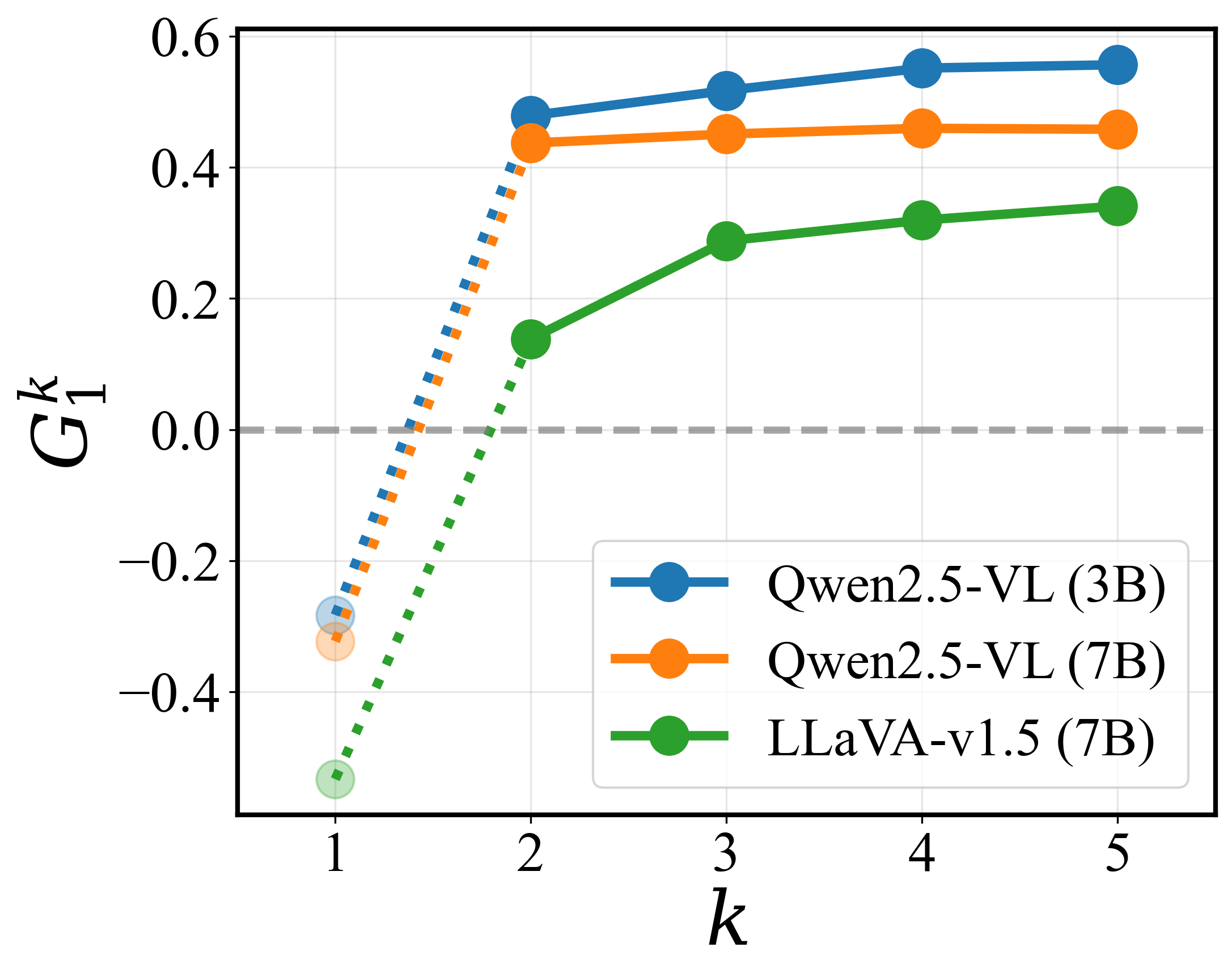}
  \caption{$G^k_1$ versus \ $k$}
  \label{fig:utility-topk-by-top1}
\end{subfigure}


\begin{subfigure}[b]{0.65\linewidth}
  \centering
  \includegraphics[width=\linewidth]{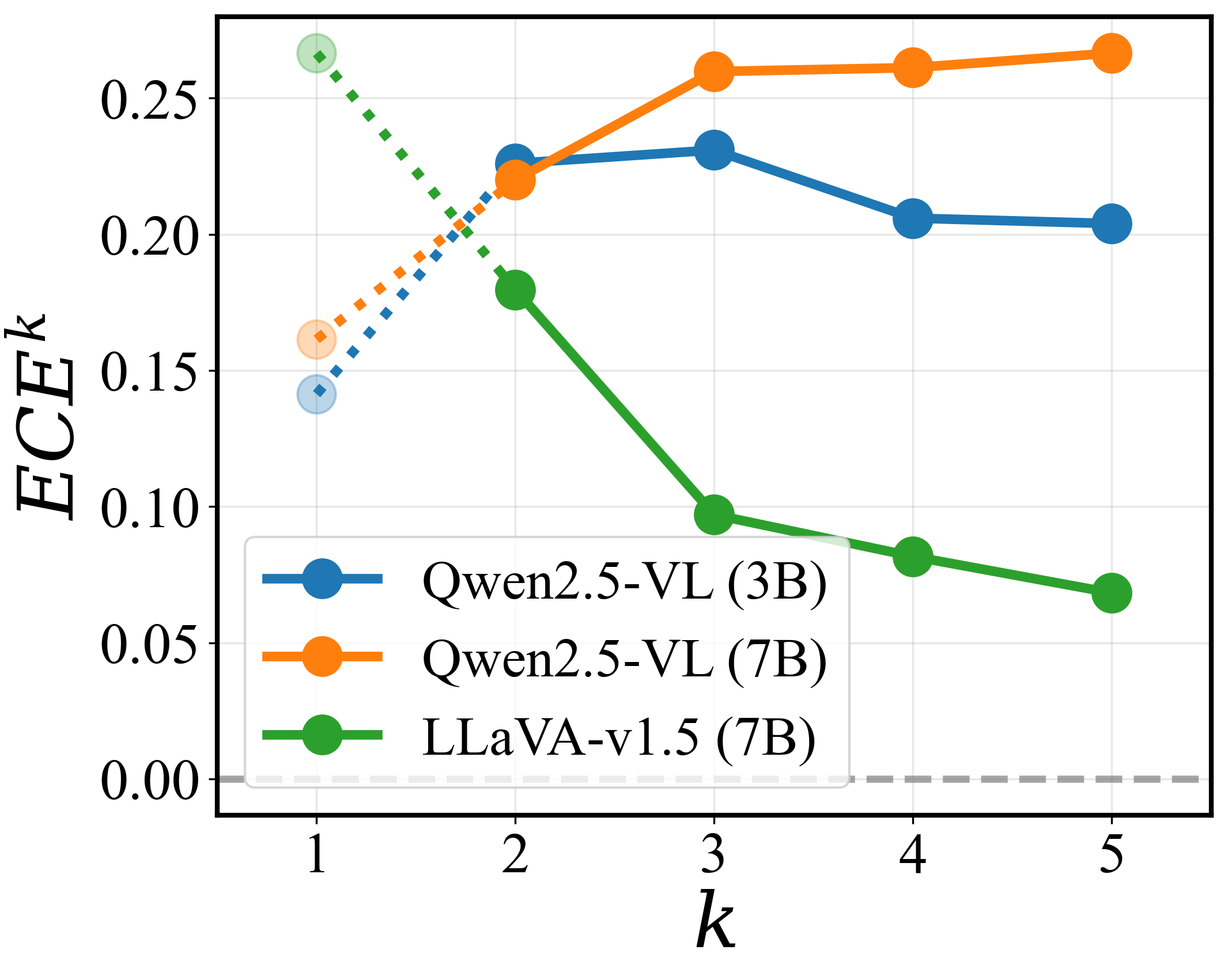}
  \caption{$\text{ECE}^k$ versus \ $k$}
  \label{fig:ece-topk}
\end{subfigure}


\caption{Empirical evidence Theorem~\ref{thr:cal} on ChartQA using Qwen2.5-VL (3B and 7B), and LLaVA-v1.5 (7B). (a) $G^k_1$ ($\alpha:=k$ for \emph{top-k}) is non-negative for all $k>1$. (b) Expected calibration errors $\text{ECE}^k$ ($\alpha:=k$ for \emph{top-k}) across different token rank $k$.}
\label{fig:empirical-evidence}
\end{figure}

\begin{figure*}[t]
  \centering
  \resizebox{\textwidth}{!}{
    \begin{minipage}{\textwidth}
      \begin{subfigure}[t]{0.3\textwidth}
        \centering
        \includegraphics[width=\linewidth]{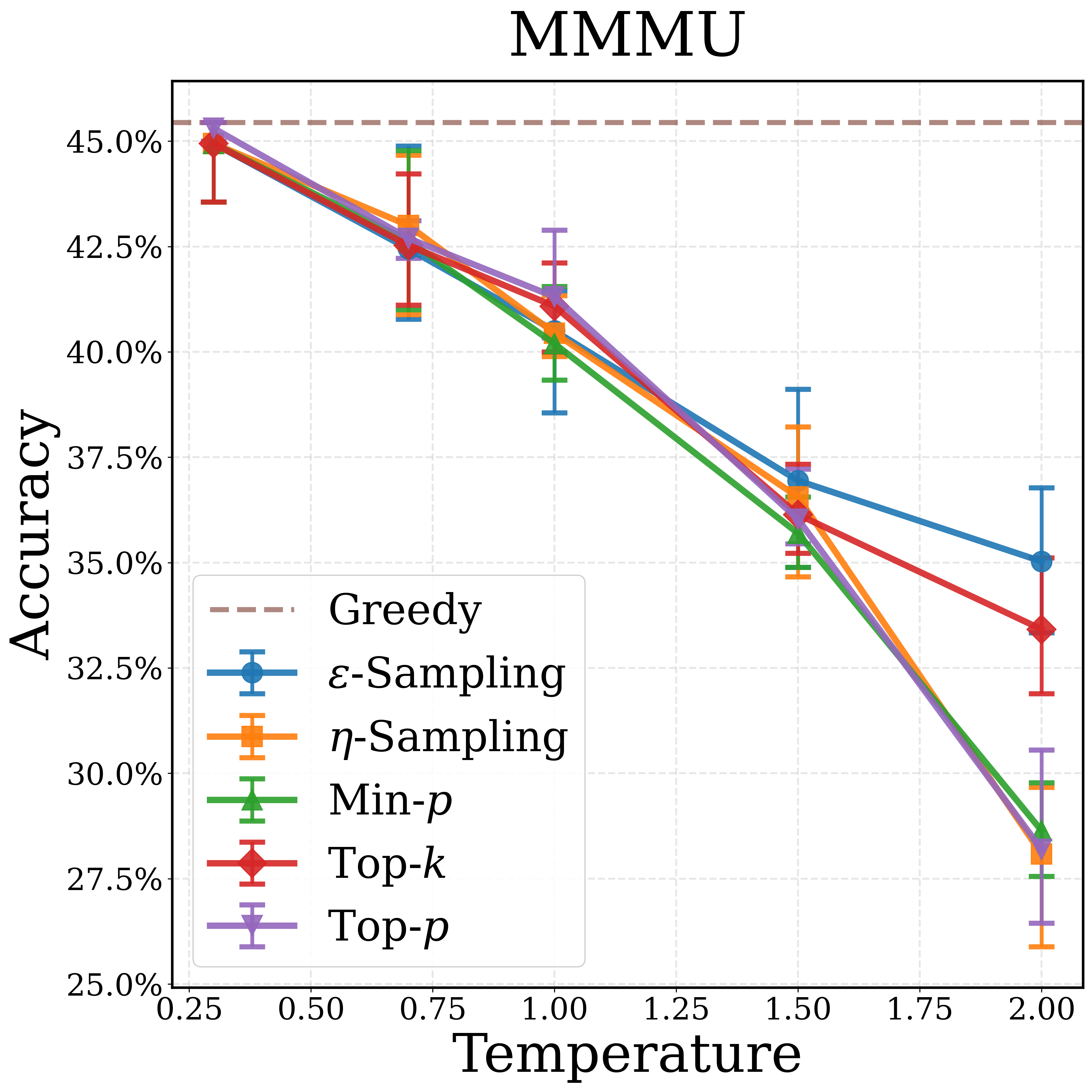}
        \caption{MMMU}
        \label{fig:three-a}
      \end{subfigure}\hfill
      \begin{subfigure}[t]{0.3\textwidth}
        \centering
        \includegraphics[width=\linewidth]{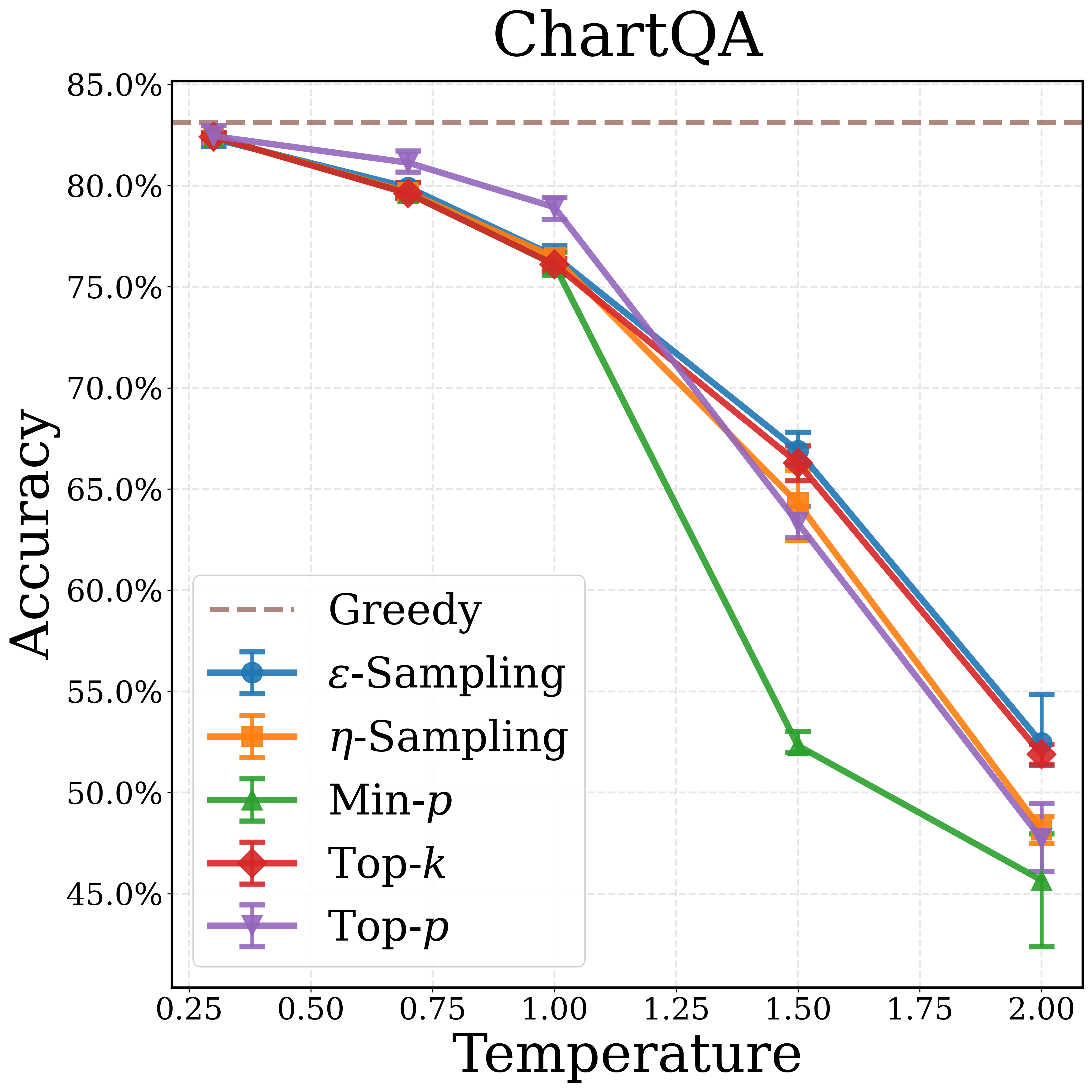}
        \caption{ChartQA}
        \label{fig:three-b}
      \end{subfigure}\hfill
      \begin{subfigure}[t]{0.3\textwidth}
        \centering
        \includegraphics[width=\linewidth]{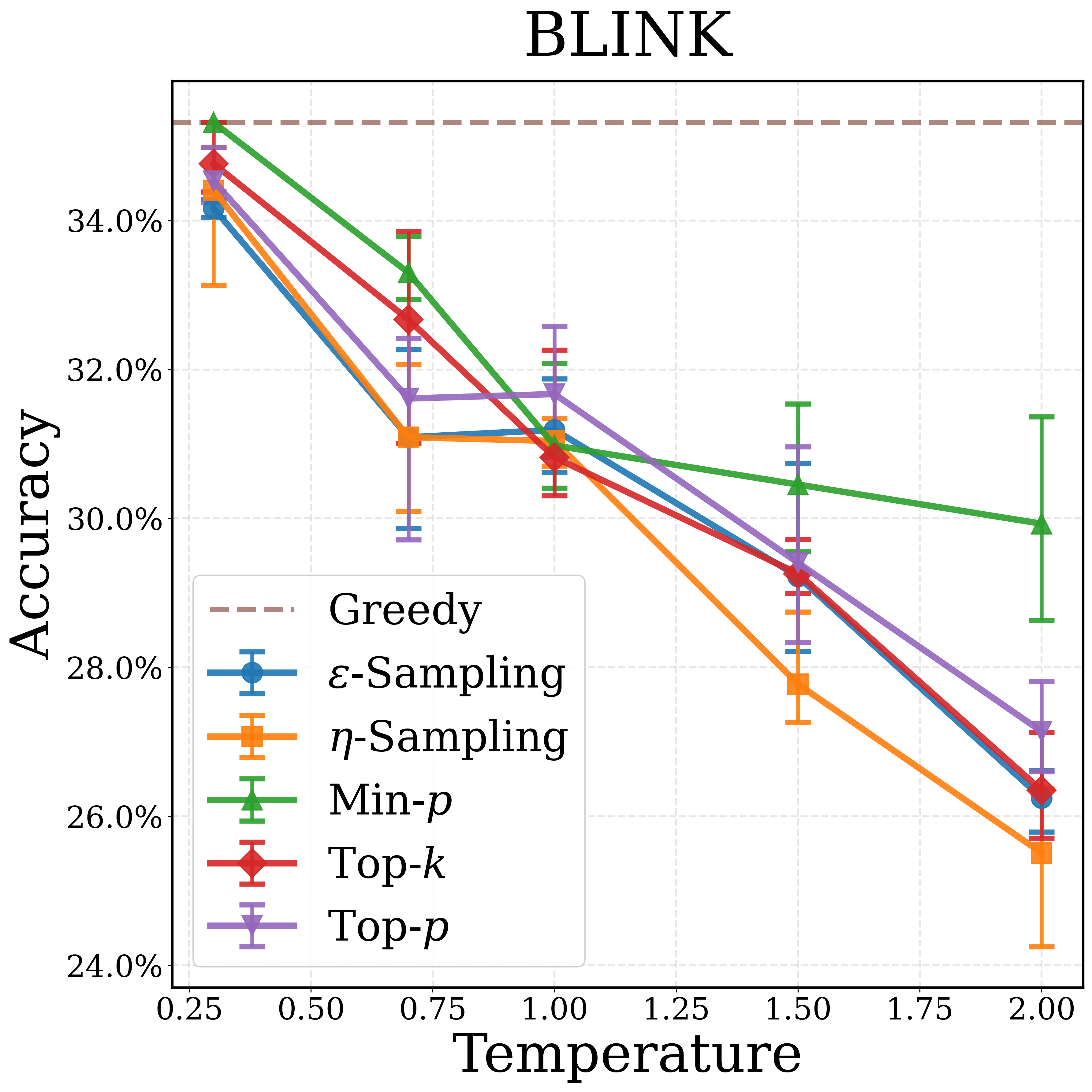}
        \caption{BLINK}
        \label{fig:three-c}
      \end{subfigure}
    \end{minipage}
  }
  \caption{Accuracy of different decoding/sampling strategies on MMMU, ChartQA and BLINK benchmarks using Qwen2.5-VL (3B) with different temperatures.}
  \label{fig:acc-temp}
\end{figure*}

\begin{table}[h]
\centering
\caption{Average Score and Hallucination Rate of different decoding/sampling strategies on MM-HallBench using Qwen2.5-VL (3B).}
\scriptsize
\begin{tabular}{lcc}
\toprule
\textbf{Strategy} & \textbf{Score $\uparrow$}  & \textbf{Hallucination Rate (\%) $\downarrow$} \\
\midrule
Temp' Only & 3.48 & 36 \\
\rowcolor{gray!10}
Top-$k$ & 3.38 & 42 \\
Top-$p$ (nucleus) & 3.58 & 36 \\
\rowcolor{gray!10}
Min-$p$ & 3.32 & 43 \\
$\varepsilon$-Sampling & 3.36 & 35 \\
\rowcolor{gray!10}
$\eta$-Sampling & 3.57 & 38 \\
Beam Search & 3.54 & \textbf{34} \\
\rowcolor{gray!10} 
Greedy & \textbf{3.64} & \textbf{34} \\
\bottomrule
\end{tabular}
\label{tab:hallucination}
\end{table}

\begin{table}[h]
\centering
\caption{Performance of different decoding strategies ($\tau=1.0$) on MMLU using Qwen2.5-VL (3B) and on CapArena using Qwen2.5-VL (7B).}
\label{tab:mmlu_caparena}
\scriptsize
\begin{tabular}{lccc}
\toprule
& \textbf{MMLU} & \multicolumn{2}{c}{\textbf{CapArena}} \\
\cmidrule(lr){2-2} \cmidrule(lr){3-4}
\textbf{Method} & \textbf{Acc (\%)} & \textbf{Avg. Score} & \textbf{Avg. Length} \\
\midrule
Temp' Only & 71.23 & -7.17 & 156.34 \\
\rowcolor{gray!10}
Top-$k$ & 71.23 & 5.17 & 158.33 \\
Top-$p$ (nucleus) & 71.23 & 2.17 & 155.06 \\
\rowcolor{gray!10}
Min-$p$ & 72.50 & 5.83 & 153.77 \\
$\varepsilon$-Sampling & 71.23 & 0.67 & 156.53 \\
\rowcolor{gray!10}
$\eta$-Sampling & 71.79 & 0.83 & 160.55 \\
\midrule
Greedy & \textbf{73.21} & \textbf{11.83} & 148.75 \\
\bottomrule
\end{tabular}
\end{table}

\begin{table*}[t]
\centering
\caption{Accuracy (in \%) of different decoding/sampling strategies on MMMU, ChartQA and BLINK benchmarks using Qwen3-VL-Thinking (4B).
Baseline: Stochastic Sampling and Beam Search.
GDRM: Greedy Decoding for Reasoning Models.
\textcolor{gray}{Greedy} is the Greedy Decoding Baseline.}
\label{tab:qwen3-4b-think-comparison}
\scriptsize
\setlength{\tabcolsep}{5pt}
\begin{tabular}{lcc cc cc}
\toprule
& \multicolumn{2}{c}{\textbf{MMMU}} & \multicolumn{2}{c}{\textbf{ChartQA}} & \multicolumn{2}{c}{\textbf{BLINK}} \\
\cmidrule(lr){2-3}\cmidrule(lr){4-5}\cmidrule(lr){6-7}
\textbf{Method} & Baseline & GDRM & Baseline & GDRM & Baseline & GDRM \\
\midrule
Temp' Only & 53.22 & 54.44 \deltaacc{+1.22} & 80.94 & 81.09 \deltaacc{+0.15} & 46.08 & 48.29 \deltaacc{+2.21} \\
\rowcolor{gray!10}
Top-$k$ & 54.11 & 54.78 \deltaacc{+0.67} & 80.78 & 81.46 \deltaacc{+0.68} & 44.77 & 47.71 \deltaacc{+2.94} \\
Top-$p$ (nucleus) & 54.56 & 54.67 \deltaacc{+0.11} & 81.56 & 81.82 \deltaacc{+0.26} & 45.13 & 45.98 \deltaacc{+0.85} \\
\rowcolor{gray!10}
Min-$p$ & 53.78 & 55.33 \deltaacc{+1.55} & 77.71 & 78.02 \deltaacc{+0.31} & 45.29 & 45.40 \deltaacc{+0.11} \\
$\varepsilon$-Sampling & 54.22 & 54.78 \deltaacc{+0.56} & 80.62 & 80.78 \deltaacc{+0.16} & 46.61 & 46.45 \deltaacc{-0.16} \\
\rowcolor{gray!10}
$\eta$-Sampling & 51.89 & 54.22 \deltaacc{+2.33} & 81.25 & 81.41 \deltaacc{+0.16} & 47.13 & 47.66 \deltaacc{+0.53} \\
Beam Search & 59.78 & 60.92 \deltaacc{+1.14} & 81.51 & 81.56 \deltaacc{+0.05} & 47.40 & 48.39 \deltaacc{+0.99} \\
\midrule
\textcolor{gray}{Greedy} & \multicolumn{2}{c}{\textcolor{gray}{54.11}} & \multicolumn{2}{c}{\textcolor{gray}{81.46}} & \multicolumn{2}{c}{\textcolor{gray}{45.71}} \\
\bottomrule
\end{tabular}
\end{table*}

\subsection{Analytical Results}
\label{sec:analytial_res}
\paragraph{Answer Distribution.}
We conduct a statistical analysis of the answer distributions across VQA and text-only QA datasets. MMMU, BLINK, and MMLU utilize a multiple-choice answer format, inducing answer distributions concentrated over specific tokens \texttt{A/B/C/D}. Although ChartQA permits more diverse responses, its answer distribution exhibits a similar highly concentrated pattern. Specifically, 75.16\% of answers are numerical. Among the non-numerical ones, the vast majority fall into limited categories: 66.04\% are single words, 23.48\% represent country names, and 11.74\% are binary responses. These results provide empirical evidence that VQA answer distributions are head-heavy, dominated by high-frequency tokens such as numbers, colors, and binary responses.

\paragraph{Empirical Evidence of Greedy Optimality.}
We demonstrate empirical evidence of Theorem~\ref{thr:cal} using the top-$k$ sampling. Figure~\ref{fig:utility-topk-by-top1} visualizes $G_1^k$ (\emph{i.e.}, $G^\alpha_1$ where $\alpha:=k$) across three models on ChartQA benchmark. We can observe that $G_1^k>0$ holds for all $k>1$. As a result, greedy decoding is provably the optimal decoding strategy under Condition \circled{1} in Theorem~\ref{thr:cal}, which is consistent with the experiment results demonstrated in Table~\ref{tab:qwen3-3b-bytemp},~\ref{tab:qwen3-7b-bytemp} and~\ref{tab:llava-7b-bytemp}. We observe similar behavior for other sampling methods. For instance, using top-$p$ sampling ($p=0.9$) with Qwen2.5-VL (3B) on ChartQA yields $G_1^p=0.173>0$ (\emph{i.e.}, $G^\alpha_1$ where $\alpha:=p$). 

\paragraph{Failure Analysis.}
We investigate the failure modes of stochastic sampling strategies by analyzing 177 cases where greedy decoding yields correct answers while stochastic decoding fails. Figure~\ref{fig:failure_mode} shows the detailed failure modes for both numeric and non-numeric cases. We can observe that numeric questions account for the majority of failure cases, with more than half of them exhibiting severe deviations ($>$20\% error) rather than minor inaccuracies. Further qualitative inspection reveals other specific failure modes, such as order-of-magnitude hallucinations (\emph{e.g.}, $2.6\to26000$), answer-type mismatch (\emph{e.g.}, $79\to$ \texttt{"Yes"}), and erroneous abstentions (\emph{e.g.}, $1.0\to$\texttt{"Unanswerable"}). 
Non-numeric failures exhibit similar tail-drift artifacts, including answer-type mismatch (\emph{e.g.}, \texttt{"Disapprove"} $\to 53$), color confusion (\emph{e.g.}, \texttt{"Light Blue"} $\to$ \texttt{"Gray"}), and logical flips (\emph{e.g.}, \texttt{"Yes"}$\to$\texttt{"No"}). These failure modes indicate that VQA distributions are head-heavy, concentrating valid answers within the high-probability region. Consequently, stochastic errors represent low-probability tail drift rather than semantically correct alternatives.

\paragraph{Hallucination.} 
We evaluate different decoding strategies on MM-HallBench using Qwen2.5-VL (3B). Table~\ref{tab:hallucination} compares answer score and hallucination rate under stochastic and deterministic decoding methods, including beam search and greedy decoding. Greedy decoding achieves the highest average answer score and the lowest hallucination rate. These results support a more precise interpretation of the effect of stochasticity: although truncation-based sampling removes part of the low-probability tail, it still assigns nonzero probability to lower-ranked candidates within the retained candidate set, which can be less calibrated or visually grounded than the top-$1$ candidate. A rank-wise analysis of the truncated candidate set provides further evidence for this behavior (Appendix~\ref{sec:appendix_chartqa}). Results reported by \citet{li2025hidden} on POPE~\cite{rohrbach2018object} provide additional corroborating evidence, with greedy decoding outperforming nucleus sampling across their evaluated models (Table~\ref{tab:acc_f1_pope}; Appendix~\ref{hall-pope}).

\paragraph{Calibration Error versus Model Performance.}
We calculate the calibration errors of MLLMs and analyze their relation with model performance. Figure~\ref{fig:ece-topk} illustrates the rank-conditioned calibration errors, denoted as $\text{ECE}^k$ ($\alpha:=k$). We observe distinct calibration patterns across models: Qwen2.5-VL (3B and 7B) exhibits strong calibration at the top-$1$ rank, with calibration degrading at larger $k$, whereas LLaVA-1.5 (7B) exhibits higher calibration error at $k=1$. On ChartQA, this poorer top-$1$ calibration corresponds to a smaller greedy advantage for LLaVA-1.5 (1.51 percentage points over the best stochastic strategy), compared with 1.99 and 1.97 points for Qwen2.5-VL 3B and 7B, respectively. More broadly, LLaVA-1.5 also provides a case in which a non-greedy strategy is preferable: on MMMU at $\tau=1.0$, top-$p$ sampling achieves 35.67\% accuracy compared with 34.78\% for greedy decoding. These observations illustrate the diagnostic role of calibration: weaker top-$1$ calibration can diminish or eliminate the greedy advantage, motivating empirical comparison of alternative strategies when the sufficient conditions in Theorem~\ref{thr:cal} are not certified.

\subsection{Text-only QA and Open-ended Generation}
\label{sec:text_qa}

Although motivated by the head-heavy nature of closed-ended VQA answer distributions, our analysis extends beyond this specific setting in two directions. First, the derivation of greedy optimality in Theorem~\ref{thr:cal} depends solely on the probabilistic structure of the answer-level posterior and remains mathematically valid when the visual input is omitted. Consequently, greedy decoding is optimal for text-only QA whenever either Condition~\circled{1} or Condition~\circled{2} holds. Evaluations on MMLU (Table~\ref{tab:mmlu_caparena}) support this extension: greedy decoding outperforms the evaluated stochastic strategies, and $G^k_1$ remains positive for all $k>1$ (Appendix~\ref{sec:appendix_mmlu}).

Second, as a separate empirical probe outside the verified scope of our theory, we evaluate detailed image captioning on CapArena~\cite{cheng2025caparena}. At $\tau=1.0$, greedy decoding achieves the highest average score while producing slightly shorter outputs (Table~\ref{tab:mmlu_caparena}). We emphasize that this is a narrow empirical result on one open-ended benchmark and do not claim that the sufficient conditions in Theorem~\ref{thr:cal} have been verified for CapArena. Extending the theoretical analysis to open-ended generation, where multiple semantically valid outputs may exist, is left for future work.

\section{Greedy Decoding for Reasoning Model}
\label{sec:gdrm}
\label{sec:thinking}
Recent advancements in MLLMs with reasoning capabilities, such as Qwen3-VL-Thinking~\cite{qwen3vl}, have demonstrated superior performance compared to non-reasoning models on various multimodal tasks. Conventionally, it is believed that introducing uncertainty is essential for the reasoning process, and thus greedy decoding is often considered suboptimal for these models~\cite{naik2024diversitythoughtimprovesreasoning}. However, motivated by Theorem~\ref{thr:cal}, we demonstrate that multimodal reasoning can still benefit from greedy decoding. Recent study shows that effective reasoning correlates with increasing token probabilities over the generation trajectory~\cite{token_signature}. Thus, while the internal reasoning process can exhibit significant stochasticity, the final outputs are often produced with high confidence (large $q(a^1|I, x)$), following sufficient deliberation. For VQA tasks with unique ground-truth answers, a high-confidence prediction that aligns with the correct answer results in a low calibration error $\text{ECE}^1$, which facilitates the satisfaction of both conditions outlined in Theorem~\ref{thr:cal}. 

Following the aforementioned theoretical insight, we propose GDRM. Specifically, by incorporating the generated reasoning trace $r$ into the input context, we transform the original reasoning VQA task $(I, x)$ into a standard prediction task $(I, x + r)$. We then apply greedy decoding to generate the final answer tokens given this augmented context. The formal definition of our proposed GDRM is given as follows:
\begin{equation*}
\begin{split}
    & \left\{
    \begin{aligned}
        & r \sim \textsc{Sampling}\!\left(p_{\theta}(\cdot \mid I, x)\right),
        \,\, \text{reasoning tokens} \\
        & \hat{y} = \textsc{Greedy}\!\left(p_{\theta}(\cdot \mid I, x + r)\right),
        \,\, \text{answer tokens}
    \end{aligned}
    \right.
\end{split}
\end{equation*}
where $r$ denotes the reasoning tokens, $\hat{y}$ denotes the answer tokens, \textsc{Sampling} denotes any given sampling strategy, and $\textsc{Greedy}$ denotes greedy decoding strategy.


We evaluate GDRM on VQA benchmarks using Qwen3-VL-Thinking (4B) model.
As shown in Table~\ref{tab:qwen3-4b-think-comparison}, GDRM consistently outperforms standard stochastic sampling across all benchmarks, notably boosting Temperature Sampling on BLINK by over 2\% by decoupling reasoning exploration from answer generation. Furthermore, it improves upon the naive greedy decoding baseline (\emph{e.g.}, reaching 60.92\% on MMMU vs.\ 54.11\% baseline), demonstrating that combining probabilistic/beam search reasoning with deterministic answering is superior to a fully deterministic approach. Additional results on Qwen3-VL-30B-A3B-Thinking (Table~\ref{tab:gdrm_30b_accuracy}; Appendix~\ref{sec:gdrm_30b}) show that GDRM consistently outperforms all stochastic baselines, demonstrating its scalability to larger models. In summary, GDRM offers a greater stability, effectively boosting the reasoning model performance by anchoring the answer tokens to the top-$1$ prediction.

\section{Conclusion}
We revisit greedy decoding for MLLMs on VQA from a calibration perspective. We demonstrate that the head-heavy answer distribution and epistemic uncertainty inherent to VQA make stochastic sampling suboptimal, as it often expands the candidate set toward low-probability distractors rather than semantically correct alternatives. Our theoretical framework formalizes the relationship between model calibration and predictive accuracy, identifying sufficient conditions under which greedy decoding is optimal. Our extensive experiments across diverse VQA benchmarks demonstrate the superior performance of greedy decoding to stochastic sampling. Furthermore, we propose GDRM, a decoding strategy that anchors final answers to top-$1$ predictions without compromising the model's underlying reasoning capabilities.

Together, our findings caution against naively inheriting LLMs decoding heuristics in MLLMs without task-specific justification and demonstrate that greedy decoding can be an efficient yet strong default for VQA. 


\section*{Limitations}

\paragraph{Observability of sufficient conditions.}
Theorem~\ref{thr:cal} is defined with respect to the ground-truth posterior $p(a\mid I,x)$, which is not directly observable in practice. We therefore estimate the required calibration quantities from labeled development set. Although our held-out experiment demonstrates that a small development set can provide a useful pre-hoc diagnostic, finite-sample estimation error and distribution shift between development and test data may affect its reliability.

\paragraph{Task and benchmark scope.}
Our primary conclusions concern closed-ended VQA with head-heavy answer distributions and predominantly epistemic uncertainty. The CapArena experiment is only a preliminary empirical probe of open-ended generation, and we do not verify the conditions of Theorem~\ref{thr:cal} in that setting. Our benchmarks are also English-language and do not cover highly specialized domains such as medical VQA; broader evaluation across languages, domains, and tasks with multiple valid answers remains future work.

\paragraph{Decoding strategy scope.}
The theoretical analysis applies to the stochastic sampling strategies parameterized by $\alpha$ and does not establish optimality over search- or ensemble-based methods such as beam search or self-consistency, which we consider only empirically or leave for future comparison. Likewise, although GDRM improves its corresponding stochastic or beam-search baseline in nearly all tested configurations, it does not outperform plain greedy decoding in every setting.

\section*{Acknowledgement}
This research was primarily supported by the ETH AI Center through an ETH AI Center doctoral fellowship to Boqi Chen and Yunke Ao. We would like to thank Xueyan Li for insightful discussions. 

\bibliography{custom}

@inproceedings{sun2024aligning,
  title={Aligning large multimodal models with factually augmented rlhf},
  author={Sun, Zhiqing and Shen, Sheng and Cao, Shengcao and Liu, Haotian and Li, Chunyuan and Shen, Yikang and Gan, Chuang and Gui, Liangyan and Wang, Yu-Xiong and Yang, Yiming and Keutzer, Kurt and Darrell, Trevor},
  booktitle={Findings of the Association for Computational Linguistics: ACL 2024},
  pages={13088--13110},
  year={2024}
}

@inproceedings{token_signature,
  title     = {Token Signature: Predicting Chain-of-Thought Gains with Token Decoding Feature in Large Language Models},
  author    = {Peijie Liu and Fengli Xu and Yong Li},
  booktitle = {International Conference on Machine Learning (ICML)},
  year      = {2025}
}

@misc{naik2024diversitythoughtimprovesreasoning,
      title={Diversity of Thought Improves Reasoning Abilities of LLMs}, 
      author={Ranjita Naik and Varun Chandrasekaran and Mert Yuksekgonul and Hamid Palangi and Besmira Nushi},
      year={2024},
      eprint={2310.07088},
      archivePrefix={arXiv},
      primaryClass={cs.CL},
      url={https://arxiv.org/abs/2310.07088}, 
}

@article{mmlu,
      title={Measuring Massive Multitask Language Understanding},
      author={Dan Hendrycks and Collin Burns and Steven Basart and Andy Zou and Mantas Mazeika and Dawn Song and Jacob Steinhardt},
      journal={Proceedings of the International Conference on Learning Representations (ICLR)},
      year={2021}
    }

@inproceedings{fu2024blink,
  title={Blink: Multimodal large language models can see but not perceive},
  author={Fu, Xingyu and Hu, Yushi and Li, Bangzheng and Feng, Yu and Wang, Haoyu and Lin, Xudong and Roth, Dan and Smith, Noah A and Ma, Wei-Chiu and Krishna, Ranjay},
  booktitle={European Conference on Computer Vision},
  pages={148--166},
  year={2024},
  organization={Springer}
}

@article{masry2022chartqa,
  title={Chartqa: A benchmark for question answering about charts with visual and logical reasoning},
  author={Masry, Ahmed and Long, Do Xuan and Tan, Jia Qing and Joty, Shafiq and Hoque, Enamul},
  journal={arXiv preprint arXiv:2203.10244},
  year={2022}
}

@inproceedings{yue2024mmmu,
  title={Mmmu: A massive multi-discipline multimodal understanding and reasoning benchmark for expert agi},
  author={Yue, Xiang and Ni, Yuansheng and Zhang, Kai and Zheng, Tianyu and Liu, Ruoqi and Zhang, Ge and Stevens, Samuel and Jiang, Dongfu and Ren, Weiming and Sun, Yuxuan and Wei, Cong and Yu, Botao and Yuan, Ruibin and Sun, Renliang and Yin, Ming and Zheng, Boyuan and Yang, Zhenzhu and Liu, Yibo and Huang, Wenhao and Sun, Huan and Su, Yu and Chen, Wenhu},
  booktitle={Proceedings of the IEEE/CVF Conference on Computer Vision and Pattern Recognition},
  pages={9556--9567},
  year={2024}
}

@article{bai2025qwen2,
  title={Qwen2. 5-vl technical report},
  author={Bai, Shuai and Chen, Keqin and Liu, Xuejing and Wang, Jialin and Ge, Wenbin and Song, Sibo and Dang, Kai and Wang, Peng and Wang, Shijie and Tang, Jun and Zhong, Humen and Zhu, Yuanzhi and Yang, Mingkun and Li, Zhaohai and Wan, Jianqiang and Wang, Pengfei and Ding, Wei and Fu, Zheren and Xu, Yiheng and Ye, Jiabo and Zhang, Xi and Xie, Tianbao and Cheng, Zesen and Zhang, Hang and Yang, Zhibo and Xu, Haiyang and Lin, Junyang},
  journal={arXiv preprint arXiv:2502.13923},
  year={2025}
}

@inproceedings{Liu2023LLaVA,
  title     = {Visual Instruction Tuning},
  author    = {Liu, Haotian and Li, Chunyuan and Wu, Qingyang and Lee, Yong Jae},
  booktitle = {Advances in Neural Information Processing Systems (NeurIPS)},
  year      = {2023}
}

@inproceedings{holtzman2019curious,
  title     = {The Curious Case of Neural Text Degeneration},
  author    = {Holtzman, Ari and Buys, Jan and Du, Li and Forbes, Maxwell and Choi, Yejin},
  booktitle = {International Conference on Learning Representations (ICLR)},
  year      = {2020},
  note      = {arXiv:1904.09751},
  url       = {https://openreview.net/forum?id=rygGQyrFvH}
}

@inproceedings{fan2018hierarchical,
  title     = {Hierarchical Neural Story Generation},
  author    = {Fan, Angela and Lewis, Mike and Dauphin, Yann N.},
  booktitle = {Proceedings of the 56th Annual Meeting of the Association for Computational Linguistics (ACL)},
  year      = {2018},
  pages     = {889--898},
  publisher = {Association for Computational Linguistics},
  doi       = {10.18653/v1/P18-1082},
  url       = {https://aclanthology.org/P18-1082/}
}

@inproceedings{hewitt2022desmoothing,
  title     = {Truncation Sampling as Language Model Desmoothing},
  author    = {Hewitt, John and Manning, Christopher D. and Liang, Percy},
  booktitle = {Findings of the Association for Computational Linguistics: EMNLP 2022},
  year      = {2022},
  pages     = {3388--3403},
  publisher = {Association for Computational Linguistics},
  url       = {https://aclanthology.org/2022.findings-emnlp.249/}
}

@article{nguyen2024minp,
  title   = {Turning Up the Heat: Min-\$p\$ Sampling for Creative and Coherent LLM Outputs},
  author  = {Nguyen, Minh Nhat and Baker, Andrew and Neo, Clement and Roush, Allen and Kirsch, Andreas and Shwartz-Ziv, Ravid},
  journal = {arXiv preprint arXiv:2407.01082},
  year    = {2024},
  url     = {https://arxiv.org/abs/2407.01082}
}

@inproceedings{antol2015vqa,
  title     = {{VQA}: Visual Question Answering},
  author    = {Antol, Stanislaw and Agrawal, Aishwarya and Lu, Jiasen and Mitchell, Margaret and Batra, Dhruv and Zitnick, C. Lawrence and Parikh, Devi},
  booktitle = {Proceedings of the IEEE International Conference on Computer Vision (ICCV)},
  year      = {2015},
  pages     = {2425--2433},
  doi       = {10.1109/ICCV.2015.279},
  url       = {https://openaccess.thecvf.com/content_iccv_2015/html/Antol_VQA_Visual_Question_ICCV_2015_paper.html}
}

@inproceedings{whitehead2022reliable,
  title     = {Reliable Visual Question Answering: Abstain Rather Than Answer Incorrectly},
  author    = {Whitehead, Spencer and Petryk, Suzanne and Shakib, Vedaad and Gonzalez, Joseph and Darrell, Trevor and Rohrbach, Anna and Rohrbach, Marcus},
  booktitle = {European Conference on Computer Vision (ECCV)},
  year      = {2022},
  url       = {https://www.ecva.net/papers/eccv_2022/papers_ECCV/papers/136960146.pdf}
}

@inproceedings{song2025good,
  title     = {The Good, The Bad, and The Greedy: Evaluation of {LLM}s Should Not Ignore Non-Determinism},
  author    = {Song, Yifan and Wang, Guoyin and Li, Sujian and Lin, Bill Yuchen},
  booktitle = {Proceedings of the 2025 Conference of the North American Chapter of the Association for Computational Linguistics: Human Language Technologies (NAACL-HLT)},
  year      = {2025},
  publisher = {Association for Computational Linguistics},
  doi       = {10.18653/v1/2025.naacl-long.211},
  url       = {https://aclanthology.org/2025.naacl-long.211/}
}

@article{Wang2023CalibrationSurvey,
  title        = {Calibration in Deep Learning: A Survey of the State‐of‐the‐Art},
  author       = {Wang, Cheng},
  journal      = {arXiv preprint},
  volume       = {arXiv:2308.01222},
  year         = {2023},
  note         = {Revised version v3, submitted 2 August 2023, last revised 10 May 2024},
  url          = {https://arxiv.org/abs/2308.01222}
}

@inproceedings{khan2024consistency,
  title={Consistency and uncertainty: Identifying unreliable responses from black-box vision-language models for selective visual question answering},
  author={Khan, Zaid and Fu, Yun},
  booktitle={Proceedings of the IEEE/CVF Conference on Computer Vision and Pattern Recognition},
  pages={10854--10863},
  year={2024}
}

@article{li2023evaluating,
  title={Evaluating object hallucination in large vision-language models},
  author={Li, Yifan and Du, Yifan and Zhou, Kun and Wang, Jinpeng and Zhao, Wayne Xin and Wen, Ji-Rong},
  journal={arXiv preprint arXiv:2305.10355},
  year={2023}
}

@inproceedings{leng2024mitigating,
  title={Mitigating object hallucinations in large vision-language models through visual contrastive decoding},
  author={Leng, Sicong and Zhang, Hang and Chen, Guanzheng and Li, Xin and Lu, Shijian and Miao, Chunyan and Bing, Lidong},
  booktitle={Proceedings of the IEEE/CVF Conference on Computer Vision and Pattern Recognition},
  pages={13872--13882},
  year={2024}
}

@inproceedings{favero2024multi,
  title={Multi-modal hallucination control by visual information grounding},
  author={Favero, Alessandro and Zancato, Luca and Trager, Matthew and Choudhary, Siddharth and Perera, Pramuditha and Achille, Alessandro and Swaminathan, Ashwin and Soatto, Stefano},
  booktitle={Proceedings of the IEEE/CVF Conference on Computer Vision and Pattern Recognition},
  pages={14303--14312},
  year={2024}
}

@inproceedings{an2025mitigating,
  title={Mitigating object hallucinations in large vision-language models with assembly of global and local attention},
  author={An, Wenbin and Tian, Feng and Leng, Sicong and Nie, Jiahao and Lin, Haonan and Wang, QianYing and Chen, Ping and Zhang, Xiaoqin and Lu, Shijian},
  booktitle={Proceedings of the Computer Vision and Pattern Recognition Conference},
  pages={29915--29926},
  year={2025}
}

@inproceedings{su2025activation,
  title={Activation Steering Decoding: Mitigating Hallucination in Large Vision-Language Models through Bidirectional Hidden State Intervention},
  author={Su, Jingran and Chen, Jingfan and Li, Hongxin and Chen, Yuntao and Qing, Li and Zhang, Zhaoxiang},
  booktitle={Proceedings of the 63rd Annual Meeting of the Association for Computational Linguistics (Volume 1: Long Papers)},
  pages={12964--12974},
  year={2025}
}

@inproceedings{guo2017calibration,
  title     = {On Calibration of Modern Neural Networks},
  author    = {Guo, Chuan and Pleiss, Geoff and Sun, Yu and Weinberger, Kilian Q.},
  booktitle = {Proceedings of the 34th International Conference on Machine Learning (ICML)},
  year      = {2017},
  editor    = {Precup, Doina and Teh, Yee Whye},
  volume    = {70},
  series    = {Proceedings of Machine Learning Research},
  pages     = {1321--1330},
  publisher = {PMLR},
  url       = {https://proceedings.mlr.press/v70/guo17a.html}
}

@article{brier1950verification,
  title     = {Verification of Forecasts Expressed in Terms of Probability},
  author    = {Brier, Glenn W.},
  journal   = {Monthly Weather Review},
  volume    = {78},
  number    = {1},
  pages     = {1--3},
  year      = {1950},
  publisher = {American Meteorological Society},
  doi       = {10.1175/1520-0493(1950)078<0001:VOFEIT>2.0.CO;2}
}

@article{li2025sample,
  title={Sample Smart, Not Hard: Correctness-First Decoding for Better Reasoning in LLMs},
  author={Li, Xueyan and Su, Guinan and Sachan, Mrinmaya and Geiping, Jonas},
  journal={arXiv preprint arXiv:2510.05987},
  year={2025}
}

@article{ghosh2024visual,
  title={Visual description grounding reduces hallucinations and boosts reasoning in lvlms},
  author={Ghosh, Sreyan and Evuru, Chandra Kiran Reddy and Kumar, Sonal and Tyagi, Utkarsh and Nieto, Oriol and Jin, Zeyu and Manocha, Dinesh},
  journal={arXiv preprint arXiv:2405.15683},
  year={2024}
}

@misc{qwen3vl,
      title={Qwen3-VL Technical Report}, 
      author={Shuai Bai and Yuxuan Cai and Ruizhe Chen and Keqin Chen and Xionghui Chen and Zesen Cheng and Lianghao Deng and Wei Ding and Chang Gao and Chunjiang Ge and Wenbin Ge and Zhifang Guo and Qidong Huang and Jie Huang and Fei Huang and Binyuan Hui and Shutong Jiang and Zhaohai Li and Mingsheng Li and Mei Li and Kaixin Li and Zicheng Lin and Junyang Lin and Xuejing Liu and Jiawei Liu and Chenglong Liu and Yang Liu and Dayiheng Liu and Shixuan Liu and Dunjie Lu and Ruilin Luo and Chenxu Lv and Rui Men and Lingchen Meng and Xuancheng Ren and Xingzhang Ren and Sibo Song and Yuchong Sun and Jun Tang and Jianhong Tu and Jianqiang Wan and Peng Wang and Pengfei Wang and Qiuyue Wang and Yuxuan Wang and Tianbao Xie and Yiheng Xu and Haiyang Xu and Jin Xu and Zhibo Yang and Mingkun Yang and Jianxin Yang and An Yang and Bowen Yu and Fei Zhang and Hang Zhang and Xi Zhang and Bo Zheng and Humen Zhong and Jingren Zhou and Fan Zhou and Jing Zhou and Yuanzhi Zhu and Ke Zhu},
      year={2025},
      eprint={2511.21631},
      archivePrefix={arXiv},
      primaryClass={cs.CV},
      url={https://arxiv.org/abs/2511.21631}, 
}

@article{li2025hidden,
  title={The hidden life of tokens: Reducing hallucination of large vision-language models via visual information steering},
  author={Li, Zhuowei and Shi, Haizhou and Gao, Yunhe and Liu, Di and Wang, Zhenting and Chen, Yuxiao and Liu, Ting and Zhao, Long and Wang, Hao and Metaxas, Dimitris N},
  journal={arXiv preprint arXiv:2502.03628},
  year={2025}
}

@inproceedings{cheng2025caparena,
  title={Caparena: Benchmarking and analyzing detailed image captioning in the llm era},
  author={Cheng, Kanzhi and Song, Wenpo and Fan, Jiaxin and Ma, Zheng and Sun, Qiushi and Xu, Fangzhi and Yan, Chenyang and Chen, Nuo and Zhang, Jianbing and Chen, Jiajun},
  booktitle={Findings of the Association for Computational Linguistics: ACL 2025},
  pages={14077--14094},
  year={2025}
}

@article{zhu2023minigpt,
  title={Minigpt-4: Enhancing vision-language understanding with advanced large language models},
  author={Zhu, Deyao and Chen, Jun and Shen, Xiaoqian and Li, Xiang and Elhoseiny, Mohamed},
  journal={arXiv preprint arXiv:2304.10592},
  year={2023}
}

@article{chen2023shikra,
  title={Shikra: Unleashing multimodal llm's referential dialogue magic},
  author={Chen, Keqin and Zhang, Zhao and Zeng, Weili and Zhang, Richong and Zhu, Feng and Zhao, Rui},
  journal={arXiv preprint arXiv:2306.15195},
  year={2023}
}

@article{dai2023instructblip,
  title={Instructblip: Towards general-purpose vision-language models with instruction tuning},
  author={Dai, Wenliang and Li, Junnan and Li, Dongxu and Tiong, Anthony and Zhao, Junqi and Wang, Weisheng and Li, Boyang and Fung, Pascale N and Hoi, Steven},
  journal={Advances in neural information processing systems},
  volume={36},
  pages={49250--49267},
  year={2023}
}

@inproceedings{rohrbach2018object,
  title={Object hallucination in image captioning},
  author={Rohrbach, Anna and Hendricks, Lisa Anne and Burns, Kaylee and Darrell, Trevor and Saenko, Kate},
  booktitle={Proceedings of the 2018 Conference on Empirical Methods in Natural Language Processing},
  pages={4035--4045},
  year={2018}
}

\appendix
\section{Theory Details}
\subsection{Calibration metrics}
\label{sec:cal_metric}

\paragraph{Expected Calibration Error (ECE)} 
ECE partitions the predicted probability space into $N$ bins, $b_n := \{i \mid q(\hat{a}_i|I_i, x_i) \in U_n\}$, where $U_n$ denotes a predefined probability interval for the $n$-th bin. It computes the weighted average of the discrepancy between empirical accuracy and average confidence within each bin:
\begin{equation}
    \text{ECE} := \sum_{n=1}^N \frac{|b_n|}{M} \left| \text{acc}(b_n) - \text{conf}(b_n) \right|,
\end{equation}
where $M$ is the total number of samples, and
\begin{align}
    \text{acc}(b_n) &:= \frac{1}{|b_n|} \sum_{i \in b_n} \mathbf{1}(\hat{a}_i = a^*_i), \\
    \text{conf}(b_n) &:= \frac{1}{|b_n|} \sum_{i \in b_n} q(\hat{a}_i | I_i, x_i).
\end{align}
Here, $\hat{a}_i$ is the model's prediction and $a^*_i$ is the ground truth label. Conceptually, $\text{acc}(b_n)$ approximates the expectation $\mathbb{E}[p(a|I,x) \mid q(a|I,x) \in U_n]$, while $\text{conf}(b_n)$ approximates $\mathbb{E}[q(a|I,x) \mid q(a|I,x) \in U_n]$. A more thorough mathematical investigation of ECE is provided in Appendix~\ref{sec:cal_metric}. 

We provide more detailed mathmatical explanation for ECE and BS in the following propositions.

\begin{proposition}
Assuming that $q(a|I, x)$ takes a unique value for each $a, I, x$, and that we have access to infinite data for every $(I, x)$.
Consider constructing one bin for each unique $q(a|I, x)$ value.
    ECE in our problem setup approximates  $\mathbb{E}_{I, x, a\sim p(\cdot|I, x)}[|q(a|I, x) - p(a|I, x)|]$ \label{prop:ece}
\end{proposition}
\begin{proof}
Since $q(a|I, x)$ is unique for each $(I, x, a)$ and we assign one bin per $q(a|I, x)$, each bin corresponds to a specific $(I, x, a)$ combination.
Then we have:
    \begin{align*}
    &\text{ECE}= \sum_{n=1}^N  \frac{|b_n|}{N} |\text{acc}(b_n) - \text{conf}(b_n)| \\
    \approx& \sum_{I, x, a}  p(I, x, a)|p(a|I, x)-q(a|I,x)| \\
    = &\mathbb{E}_{I, x, a\sim p(\cdot|I, x)}[|q(a|I, x) - p(a|I, x)|].
\end{align*}
\end{proof}

\paragraph{Brier Score (BS)}
The Brier Score measures the mean squared error between the predicted probability and the actual label. While the classical Brier Score assumes deterministic labels, we can derive an adapted version that accounts for the discrepancy between distributions using a binning approach:
\begin{equation}
    \text{BS} := \sum_{n=1}^N \frac{|b_n|}{M} \frac{|\text{acc}(b_n) - \text{conf}(b_n)|^2}{\text{acc}(b_n)}.
\end{equation}
This metric approximates the expected squared difference between the predicted and true probabilities across the answer space, averaged over the dataset. A more thorough mathematical investigation of BS is also provided in Appendix~\ref{sec:cal_metric}. 

To analyze the calibration of a specific parameterized strategy $\alpha$, we extend the above metrics into a continuous expectation framework.

\begin{proposition}
    Assuming that $q(a|I, x)$ takes a unique value for each $a, I, x$, and that we have access to infinite data for every $(I, x)$.
Consider constructing one bin for each unique $q(a|I, x)$ value. 
Then BS in our problem setup approximates  $\mathbb{E}_{I, x}[\sum_{a}|q(a|I, x) - p(a|I, x)|^2]$.
\end{proposition}
\begin{proof}
    The proof follows similar steps to the proof of Proposition~\ref{prop:ece}:
    \begin{align*}
        &\text{BS}\frac{1}{N} \sum_{n=1}^N |b_n| \frac{|\text{acc}(b_n) - \text{conf}(b_n)|^2}{\text{acc}(b_n)}\\
        \approx& \sum_{I, x, a}  p(I, x, a)\frac{|q(a|I, x_i) - p(a|I, x)|^2}{p(a|I, x)} \ \\
        \approx & \mathbb{E}_{I, x, a\sim p(\cdot|I, x)}\left[\frac{|q(a|I, x_i) - p(a|I, x)|^2}{p(a|I, x)}\right] \\
        = &\mathbb{E}_{I, x}[\sum_{a}|q(a|I, x) - p(a|I, x)|^2]
    \end{align*}
\end{proof}
These results motivate our definition of $ECE^\alpha$ and $BS^\alpha$ with $p$ and $q$.

\subsection{Strategy-specific Brier Score}
\label{supp-bsk}
The strategy-specific Brier Score is defined as the expected normalized squared difference:
\begin{align}
    \text{BS}^\alpha &:= \mathbb{E}_{I, x, a \sim q^\alpha(\cdot|I, x)} \left[ \frac{|q(a|I, x) - p(a|I, x)|^2}{q(a|I, x)} \right], \label{def:bs_a_app}\\
    \text{BS}^1 &:= \mathbb{E}_{I, x} \left[ \frac{|q(a^1|I, x) - p(a^1|I, x)|^2}{q(a^1|I, x)} \right],\label{def:bs_1_app}
\end{align}
where $\text{BS}^1$ is the calibration error of the greedy strategy (always taking the token with the highest confidence). 

In the definition, we substitute the denominator $p(a|I, x)$ in the original definition of BS (Appendix~\ref{sec:cal_metric}) with $q(a|I, x)$. Since the denominator serves only as a positive weighting over answers, both of them yield a valid calibration error, differing only in how deviations are weighted across answers. We choose $q(a|I, x)$ for two reasons. First, it is known and strictly positive for top-$k$ answers, making the metric stable and directly estimable, whereas $p(a|I, x)$ is unknown and may be arbitrarily small or zero. Second, it enables algebraic decomposition in the proof of Theorem~\ref{thr:cal} (see Appendix~\ref{sec:proof} for details): the bound leading to Eq.~\ref{cond:2} naturally yields a $\mathbb{E}[|q-p|^2/q]$ term (\emph{i.e.}, our $\text{BS}^\alpha$), whereas utilizing $p(a|I, x)$ would break this tractable bound and obscure the resulting sufficient condition.

\subsection{Proof of Theorem~\ref{thr:cal}}
\label{sec:proof}
\begin{proof}
We first introduce the correctness metric ($C^\alpha$) for the strategy with the parameter $\alpha$:
\begin{align}
    C^\alpha:= \mathbb{E}_{I,x,a\sim q^\alpha(\cdot|I, x)}[p(a|I,x)] \label{def:c_k}
\end{align}
    We start by bounding $C^\alpha$ using $\text{ECE}^\alpha$ and $q$:
    \begin{align*}
        &C^\alpha=\mathbb{E}_{I,x,a\sim q^\alpha(\cdot|I, x)}[p(a|I,x)] \\
        =&\mathbb{E}_{I,x,,a\sim q^\alpha(\cdot|I, x)}[q(a|I,x) \\
        &+ (p(a|I,x)-q(a|I,x))] \\
        \leq &\mathbb{E}_{I,x,a\sim q^\alpha(\cdot|I, x)}[q(a|I,x) \\
        &+ |p(a|I,x)-q(a|I,x)|] \\
        \leq &\mathbb{E}_{I,x,a\sim q^\alpha(\cdot|I, x)}[q(a|I,x)] + \text{ECE}^\alpha
    \end{align*}
    On the other hand, we have
    \begin{align*}
        C^\alpha=&\mathbb{E}_{I,x,,a\sim q^\alpha(\cdot|I, x)}[q(a|I,x) \\
        &+ (p(a|I,x)-q(a|I,x))] \\
        \geq &\mathbb{E}_{I,x,a\sim q^\alpha(\cdot|I, x)}[q(a|I,x) 
        \\&- |p(a|I,x)-q(a|I,x)|] \\
        = &\mathbb{E}_{I,x,a\sim q^\alpha(\cdot|I, x)}[q(a|I,x)] - \text{ECE}^\alpha
    \end{align*}
    Similarly, we can have $C^1\leq \mathbb{E}_{I,x}[q(a^1|I,x)] + \text{ECE}^1$ and $C^1\geq \mathbb{E}_{I,x}[q(a^1|I,x)] - \text{ECE}^1$. 
    Applying this to inequality~(\ref{cond:1}), we obtain:
    \begin{align*}
        &C^\alpha \leq \mathbb{E}_{I,x,a\sim q^\alpha(\cdot|I, x)}[q(a|I,x)] + \text{ECE}^\alpha \\ \leq & \mathbb{E}_{I,x}[q(a^1|I,x)] - \text{ECE}^1 \leq C^1
    \end{align*}
    On the other hand, we can also bound $C^\alpha$ as follows:
    \begin{align*}
        &C^\alpha = \mathbb{E}_{I,x,a\sim q^\alpha(\cdot|I, x)}\left[\frac{p(a|I,x)q(a|I,x)}{q(a|I,x)}\right] \\
        = &\mathbb{E}_{I,x,a\sim q^\alpha(\cdot|I, x)}\left[\frac{p(a|I,x)^2 + q(a|I,x)^2}{2q(a|I,x)}\right. \\
        - &\left.\frac{|p(a|I,x) - q(a|I,x)|^2}{2q(a|I,x)}  \right] \\
        \leq & \mathbb{E}_{I,x,a\sim q^\alpha(\cdot|I, x)}\left[\frac{1 + q(a|I,x)^2}{2q(a|I,x)}\right] - \frac{\text{BS}^\alpha}{2}
    \end{align*}
   Combining this bound with the previous bound on $C^1$ using $\text{ECE}^1$, and assuming inequality~(\ref{cond:2}) holds, we get: 
    \begin{align*}
        \,\,&C^\alpha \leq \mathbb{E}_{I,x,a\sim q^\alpha(\cdot|I, x)}\left[\frac{1 + q(a^\alpha|I,x)^2}{2q(a^\alpha|I,x)}\right] - \frac{\text{BS}^\alpha}{2} \\
        \leq\,\,&\mathbb{E}_{I,x}[q(a^1|I,x)] - \text{ECE}^1  \leq C^1
    \end{align*}
    The inequality $C^1\geq C^\alpha$, for all $\alpha$ that gives non-greedy strategy suggests that greedy strategy is optimal.  
\end{proof}

\section{Implementation Details}
\label{sec:implementation_detail}

\subsection{Prompt and Formatting}
\label{sec:appendix_prompt_format}
For each dataset, we provide a global system prompt.
All models use a standardized instruction template with explicit role and output style cues. 
For multiple-choice questions, we append the choices and instruct the model to answer with the option letter only (\emph{e.g.,} \texttt{A/B/C/D}). 
For free-form questions, we instruct the model to answer with a short phrase. 
We disable chain-of-thought style prompting and request only to output the answers without reasoning.

\subsection{Inference setting}
\label{sec:appendix_inference}
Inference is run on a single A6000 with 48GB memory for all models. We cap decoding at 64 tokens for free-form answers and 8 tokens for multiple-choice questions. 

\subsection{Hyperparameters}
\label{sec:appendix_hyperparameters}
We search: $T \in \{0.3, 0.7, 1.0, 1.5, 2.0\}$, $k \in \{20, 50\}$, $p \in \{0.9, 0.95\}$, $p_{base} \in \{0.05, 0.1\}$, $\varepsilon \in \{0.0009, 0.0006\}$, and $\eta \in \{0.0006, 0.0002\}$.
For the Qwen3-VL Thinking series models, we use the officially recommended decoding hyperparameters detailed in Table~\ref{tab:qwen_thinking_params}.
\begin{table}[h]
\centering
\caption{Hyperparameters used for Qwen3-VL Thinking models.}
\begin{tabular}{ll}
\toprule
\textbf{Parameter} & \textbf{Value} \\
\midrule
temperature & 1.0 \\
top-$p$ & 0.95 \\
top-$k$ & 20 \\
\bottomrule
\end{tabular}
\label{tab:qwen_thinking_params}
\end{table}

\subsection{System Prompts}
\label{sec:appendix_sys_prompts}

This section provides the exact system prompts utilized for each VQA benchmark during our evaluation. These prompts were designed to constrain the model's output to the specific format required by each dataset's evaluation script.

\paragraph{BLINK}
For the BLINK benchmark, the following system prompt was used to ensure the model responded with only the correct option letter(s).

\begin{lstlisting}[style=appendixprompt]
You are being evaluated on BLINK. 
Answer with ONLY the option letter(s) for the correct choice (A/B/C/...).
Do not output any words, punctuation, or explanation.
\end{lstlisting}

\paragraph{ChartQA}
For the ChartQA benchmark, we employed a more detailed system prompt to guide the model in generating answers in various formats (\emph{e.g.}, numeric, "Yes/No", multiple-choice).

\begin{lstlisting}[style=appendixprompt]
You are a vision--language model specialized in chart question answering.
Answer with ONLY the final result (no explanation).
Rules:
For numeric answers: use digits and include the unit exactly as shown in the chart (e.g., %, <deg>C, $, K).
For yes/no: reply "Yes" or "No" only.
For multiple-choice: output the option letter(s) only (e.g., "B" or "A,C") with no spaces.
If the question cannot be answered from the chart, reply exactly: "unanswerable".
    }%
\end{lstlisting}

\paragraph{MMMU}
No system prompt was used for the MMMU benchmark. The evaluations were without any system prompt being passed to the models.

\subsection{Metrics}
\label{metric}
\paragraph{Hallucination Benchmark.} The Answer Score is a GPT-4o-assigned rating on a 0--6 scale that jointly reflects response informativeness and hallucination status, following the evaluation protocol of~\citet{sun2024aligning}. The Hallucination Rate is the fraction of responses assigned a hallucination-containing rating under the same protocol.  

The detailed evaluation prompt is as follows:

\begin{lstlisting}[style=appendixprompt]
Please act as an impartial and objective judge and evaluate the quality of the response provided by a Large Multimodal Model (LMM) to the user question. Your evaluation should be mainly based on whether the response is informative, and whether the response contains any hallucination. Hallucination, in this context, refers to a situation where the LMM generates a response that includes information not present or implied in the image or previous conversation. A hallucination could be a false claim about an object, action, emotion, or any other detail that is not grounded in the image.

For clarity, consider these examples:
...

With these examples in mind, please help me evaluate whether the response by the LMM is informative, and whether hallucination exists in it, based on the comparison between the LMM's response and the factual information provided in the image contents, question, and the standard human-generated answer below.

Please note that the standard human-generated answer may only contain factual information but may not give a detailed analysis. Also, the standard human-generated answer may not be completely comprehensive in describing all the objects and their attributes, so please be a bit more cautious during evalutation. LMM's detailed analysis or reasoning should be encouraged.

To evaluate the LMM responses, first, begin your evaluation by providing a short explanation. Second, after providing your explanation, you must rate the response by choosing from the following options:

- Rating: 6, very informative with good analysis or reasoning, no hallucination
- Rating: 5, very informative, no hallucination
- Rating: 4, somewhat informative, no hallucination
- Rating: 3, not informative, no hallucination
- Rating: 2, very informative, with hallucination
- Rating: 1, somewhat informative, with hallucination
- Rating: 0, not informative, with hallucination
    }%
\end{lstlisting}

\paragraph{Open-ended Image Captioning.}
The Average Score is a GPT-4o-as-a-judge measurement obtained via pairwise battles between the evaluated model and three reference captioners (GPT-4o, CogVLM2, and MiniCPM-V) over 600 images, computed as the mean of the three per-baseline net win rates (in percent, range $[-100, 100]$).
The Average Length is the mean length of the generated captions, reported alongside the Average Score to contextualize verbosity.
Both metrics follow the CapArena-Auto evaluation protocol established in prior work~\cite{cheng2025caparena}.

\section{Additional Results}
\subsection{Additional results on the ChartQA dataset}
\label{sec:appendix_chartqa}

\paragraph{Pre-hoc Decoding Strategy Selection.} To evaluate whether Theorem~\ref{thr:cal} can guide decoding-strategy selection before observing target-set performance, we conduct a pre-hoc experiment using Qwen2.5-VL (3B) on ChartQA. We randomly sample 200 questions from the ChartQA training split as a held-out development set and estimate $G_1^k$ for top-$k$ sampling. All main results are evaluated on the official validation split, so no validation examples are used when computing this diagnostic.

\begin{table}[h]
\centering
\small
\begin{tabular}{cc}
\toprule
\textbf{$k$} & \textbf{Estimated $G_1^k$} \\
\midrule
\textcolor{gray}{1} & \textcolor{gray}{-0.5455} \\
\midrule
\rowcolor{gray!10}
2 & 0.0470 \\
3 & 0.1856 \\
\rowcolor{gray!10}
4 & 0.1965 \\
5 & 0.2752 \\
\bottomrule
\end{tabular}
\caption{Pre-hoc estimates of $G_1^k$ on a held-out development set of 200 ChartQA training examples using Qwen2.5-VL (3B). Theorem~\ref{thr:cal} applies to the non-greedy alternatives $k>1$; $k=1$ is shown only for completeness.}
\label{tab:prehoc_g1}
\end{table}

For every evaluated non-greedy alternative $k\in\{2,3,4,5\}$, we obtain $G_1^k>0$. Condition~\circled{1} of Theorem~\ref{thr:cal} therefore predicts, using only the held-out development set, that greedy decoding should outperform these top-$k$ alternatives. This prediction is confirmed on the disjoint ChartQA validation set, where greedy decoding achieves the best accuracy (Table~\ref{tab:qwen3-3b-bytemp}). This experiment demonstrates how the theorem can be used as a \emph{pre-hoc diagnostic} rather than only as a post-hoc explanation.

\paragraph{Rank-wise Analysis within the Truncated Candidate Set} To characterize the answer distribution retained after truncation, we conduct a rank-wise analysis on ChartQA using Qwen2.5-VL (3B). For each of the 1,920 examples, we consider the top-$5$ candidate answers and report two quantities at each rank: (i) the mean model confidence after renormalizing probability mass within the top-$5$ candidate set, and (ii) the empirical probability that the candidate at that rank is the ground-truth answer.

\begin{table}[h]
\centering
\small
\begin{tabular}{ccc}
\toprule
\textbf{$k$} & \textbf{Mean Norm. Confidence} & \textbf{Empirical Correctness} \\
\midrule
1 & 0.8756 & 0.8617 \\
\rowcolor{gray!10}
2 & 0.0779 & 0.0739 \\
3 & 0.0251 & 0.0230 \\
\rowcolor{gray!10}
4 & 0.0132 & 0.0119 \\
5 & 0.0081 & 0.0072 \\
\bottomrule
\end{tabular}
\caption{Rank-wise distribution within the top-$5$ truncated candidate set on ChartQA using Qwen2.5-VL (3B), with $n=1{,}920$ examples.}
\label{tab:truncated_rank}
\end{table}

Correctness is overwhelmingly concentrated on the top-$1$ candidate and decreases sharply with rank, closely tracking the model's confidence distribution. Thus, even after the low-probability tail is truncated, stochastic decoding assigns nonzero probability to retained lower-ranked candidates that are substantially less likely to be correct. This complements the rank-conditioned calibration results in Figure~\ref{fig:ece-topk} and the failure analysis in Figure~\ref{fig:failure_mode}.

\subsection{Additional results on the MMLU dataset}
\label{sec:appendix_mmlu}
\paragraph{Empirical evidence on the MMLU dataset.} Fig~\ref{fig:mmlu} shows that the $G^k_1$ remains positive for all $k > 1$ on the MMLU dataset using Qwen2.5-VL (3B).
Theorem~\ref{thr:cal} proves that in this case, greedy decoding is optimal.

\begin{figure}[t]
\centering
\includegraphics[width=0.9\linewidth]{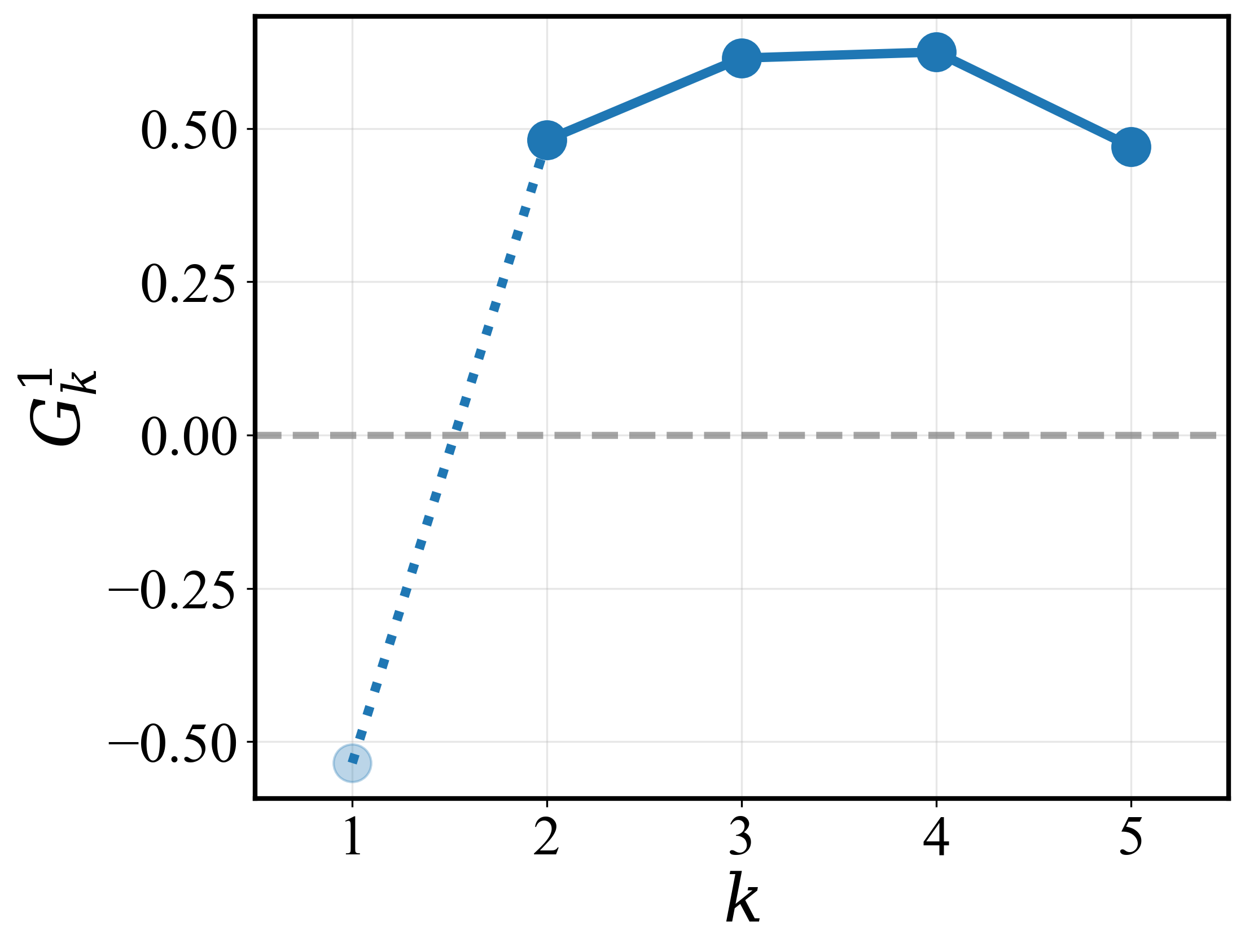}
\caption{Empirical evidence ($G^k_1$ vs. $k$) for Theorem~\ref{thr:cal} on MMLU using Qwen2.5-VL (3B)
}
\label{fig:mmlu}
\end{figure}

\paragraph{Results with different temperatures.}
Table~\ref{tab:mmlu-appendix} provides additional results using Qwen2.5-VL (3B) across different temperature settings. The results demonstrate that although lowering the temperature slightly mitigates the negative effects of stochasticity, greedy decoding consistently outperforms stochastic sampling methods.

\begin{table}[hbpt]
\centering
\caption{Accuracy (in \%) of different decoding/sampling strategies on MMLU benchmark using Qwen2.5-VL (3B)..}
\label{tab:mmlu-appendix}
\small
\begin{tabular}{lcc}
\toprule
& \multicolumn{2}{c}{\textbf{MMLU}} \\
\cmidrule(lr){2-3}
\textbf{Method} & $\tau{=}0.7$ & $\tau{=}2.0$ \\
\midrule
Temp' Only & 72.86 & 71.93 \\
\rowcolor{gray!10}
Top-$k$ & 71.23 & 71.23 \\
Top-$p$ (nucleus) & 71.23 & 69.47 \\
\rowcolor{gray!10}
Min-$p$ & 71.58 & 70.18 \\
$\varepsilon$-Sampling & 71.23 & 68.77 \\
\rowcolor{gray!10}
$\eta$-Sampling & 70.36 & 71.79 \\
\midrule
Greedy & \multicolumn{2}{c}{\textbf{73.21}} \\
\bottomrule
\end{tabular}
\end{table}

\subsection{Detailed results on the POPE Benchmark}
\label{hall-pope}
Table~\ref{tab:acc_f1_pope} shows average Accuracy and F1 scores on the POPE benchmark~\cite{rohrbach2018object} using greedy decoding and stochastic sampling (nucleus) strategies. Greedy (vanilla) decoding consistently outperforms nucleus sampling (vanilla) across all four evaluated models on both metrics~\cite{li2025hidden}. The performance gap is particularly pronounced in models like MiniGPT-4~\cite{zhu2023minigpt}, where nucleus sampling suffers a severe degradation in both Accuracy (-16.2\%) and F1 score (-14.78\%). These results are consistent with our findings on the MM-HallBench dataset in the main text, providing additional empirical evidence that stochastic sampling expands the candidate set into the low-probability tail, frequently associated with ambiguous visual evidence and therefore at a higher risk of hallucination.
\begin{table}[h]
\centering
\caption{Comparison of average Accuracy and F1 scores, averaged across
random, popular, and adversarial object splits, between Greedy and Nucleus sampling across different models on POPE benchmark. Results are adopted from~\citet{li2025hidden}. Best results are in \textbf{bold}.}
\label{tab:acc_f1_pope}
\scriptsize
\setlength{\tabcolsep}{5pt}
\begin{tabular}{lcccc}
\toprule
& \multicolumn{2}{c}{\textbf{Avg. Accuracy}} & \multicolumn{2}{c}{\textbf{Avg. F1}} \\
\cmidrule(lr){2-3} \cmidrule(lr){4-5}
\multirow{-2}{*}{\textbf{Model}} & Greedy & Nucleus & Greedy & Nucleus \\
\midrule
LLaVA-1.5~\cite{Liu2023LLaVA} & \textbf{84.79} & 81.26 & \textbf{85.61} & 82.40 \\
\rowcolor{gray!10}
MiniGPT-4~\cite{zhu2023minigpt} & \textbf{76.76} & 60.56 & \textbf{76.82} & 62.04 \\
Shikra~\cite{chen2023shikra} & \textbf{81.32} & 78.94 & \textbf{82.01} & 80.18 \\
\rowcolor{gray!10}
InstructBLIP~\cite{dai2023instructblip} & \textbf{84.36} & 78.83 & \textbf{84.64} & 79.74 \\
\bottomrule
\end{tabular}
\end{table}

\subsection{Greedy Decoding for Reasoning Model}
\label{sec:gdrm_30b}
To evaluate the scalability of Greedy Decoding for Reasoning Model (GDRM), we conduct additional experiments on ChartQA using Qwen3-VL-30B-A3B-Thinking, a larger Mixture-of-Experts reasoning model comprising 30B total parameters (3B active). As shown in Table~\ref{tab:gdrm_30b_accuracy}, GDRM consistently outperforms all stochastic sampling baselines, suggesting that its effectiveness scales to larger-scale models.

\begin{table}[h]
\centering
\small
\setlength{\tabcolsep}{5pt}
\caption{Comparison of Baseline and GDRM accuracy across different sampling strategies for Qwen3-VL-30B-A3B-Thinking.}
\label{tab:gdrm_30b_accuracy}
\begin{tabular}{lcc}
\toprule
& \multicolumn{2}{c}{\textbf{ChartQA}} \\ 
\cmidrule(lr){2-3}
\textbf{Method} & Baseline & GDRM \\
\midrule
Top-$k$ & 84.68 & 85.40 \deltaacc{+0.72} \\
\rowcolor{gray!10}
Top-$p$ (nucleus) & 84.15 & 84.70 \deltaacc{+0.55} \\
$\varepsilon$-Sampling & 84.22 & 84.86 \deltaacc{+0.64} \\
\bottomrule
\end{tabular}
\end{table}

\subsection{Detailed results with different random seeds}
\label{sec:appendix_seeds}
We list the evaluation results with different random seeds using Qwen2.5-VL (3B) from Table~\ref{tab:mmmu-seed42} to Table~\ref{tab:blink-seed789}.

\begin{table*}[htbp]
\centering
\caption{Accuracy (in \%) of different decoding/sampling strategies on MMMU benchmark (seed=42) using Qwen2.5-VL (3B).}
\label{tab:mmmu-seed42}
\small
\setlength{\tabcolsep}{6pt}
\begin{tabular}{lccccc}
\toprule
\textbf{Method} & $\tau{=}0.3$ & $\tau{=}0.7$ & $\tau{=}1.0$ & $\tau{=}1.5$ & $\tau{=}2.0$ \\
\midrule
Top-$k$ & 45.33 & 44.22 & 42.11 & 37.33 & 33.89 \\
\rowcolor{gray!10}
Top-$p$ (nucleus) & 45.44 & 42.44 & 40.33 & 35.44 & 28.67 \\
Min-$p$ & 45.44 & 42.56 & 39.89 & 35.22 & 27.56 \\
\rowcolor{gray!10}
$\varepsilon$-Sampling & 45.44 & 43.00 & 41.44 & 38.56 & 34.56 \\
$\eta$-Sampling & 45.44 & 44.22 & 40.33 & 38.22 & 29.22 \\
\midrule
Greedy& \multicolumn{5}{c}{45.44} \\
\bottomrule
\end{tabular}
\end{table*}

\begin{table*}[htbp]
\centering
\caption{Accuracy (in \%) of different decoding/sampling strategies on MMMU benchmark (seed=123) using Qwen2.5-VL (3B).}
\label{tab:mmmu-seed123}
\small
\setlength{\tabcolsep}{6pt}
\begin{tabular}{lccccc}
\toprule
\textbf{Method} & $\tau{=}0.3$ & $\tau{=}0.7$ & $\tau{=}1.0$ & $\tau{=}1.5$ & $\tau{=}2.0$ \\
\midrule
Top-$k$ & 45.33 & 41.11 & 41.22 & 35.22 & 32.78 \\
\rowcolor{gray!10}
Top-$p$ (nucleus) & 45.00 & 43.11 & 42.89 & 37.22 & 30.56 \\
Min-$p$ & 45.33 & 41.00 & 39.33 & 36.56 & 29.78 \\
\rowcolor{gray!10}
$\varepsilon$-Sampling & 45.33 & 40.78 & 38.56 & 34.89 & 33.33 \\
$\eta$-Sampling & 45.44 & 40.89 & 39.89 & 34.67 & 29.67 \\
\midrule
Greedy& \multicolumn{5}{c}{45.44} \\
\bottomrule
\end{tabular}
\end{table*}

\begin{table*}[htbp]
\centering
\caption{Accuracy (in \%) of different decoding/sampling strategies on MMMU benchmark (seed=456) using Qwen2.5-VL (3B).}
\label{tab:mmmu-seed456}
\small
\setlength{\tabcolsep}{6pt}
\begin{tabular}{lccccc}
\toprule
\textbf{Method} & $\tau{=}0.3$ & $\tau{=}0.7$ & $\tau{=}1.0$ & $\tau{=}1.5$ & $\tau{=}2.0$ \\
\midrule
Top-$k$ & 43.56 & 42.56 & 40.00 & 35.78 & 31.89 \\
\rowcolor{gray!10}
Top-$p$ (nucleus) & 45.44 & 42.22 & 41.33 & 35.78 & 27.11 \\
Min-$p$ & 43.56 & 44.78 & 41.56 & 36.11 & 29.11 \\
\rowcolor{gray!10}
$\varepsilon$-Sampling & 43.56 & 44.89 & 41.33 & 35.22 & 35.44 \\
$\eta$-Sampling & 43.56 & 44.67 & 41.33 & 37.67 & 27.56 \\
\midrule
Greedy& \multicolumn{5}{c}{45.44} \\
\bottomrule
\end{tabular}
\end{table*}

\begin{table*}[htbp]
\centering
\caption{Accuracy (in \%) of different decoding/sampling strategies on MMMU benchmark (seed=789) using Qwen2.5-VL (3B).}
\label{tab:mmmu-seed789}
\small
\setlength{\tabcolsep}{6pt}
\begin{tabular}{lccccc}
\toprule
\textbf{Method} & $\tau{=}0.3$ & $\tau{=}0.7$ & $\tau{=}1.0$ & $\tau{=}1.5$ & $\tau{=}2.0$ \\
\midrule
Top-$k$ & 45.44 & 42.22 & 41.00 & 36.22 & 35.11 \\
\rowcolor{gray!10}
Top-$p$ (nucleus) & 45.33 & 43.00 & 40.67 & 35.67 & 26.44 \\
Min-$p$ & 45.44 & 42.22 & 40.00 & 34.89 & 28.11 \\
\rowcolor{gray!10}
$\varepsilon$-Sampling & 45.44 & 41.11 & 40.67 & 39.11 & 36.78 \\
$\eta$-Sampling & 45.33 & 42.22 & 40.22 & 35.67 & 25.89 \\
\midrule
Greedy& \multicolumn{5}{c}{45.44} \\
\bottomrule
\end{tabular}
\end{table*}

\begin{table*}[htbp]
\centering
\caption{Accuracy (in \%) of different decoding/sampling strategies on ChartQA benchmark (seed=42) using Qwen2.5-VL (3B).}
\label{tab:chartqa-seed42}
\small
\setlength{\tabcolsep}{6pt}
\begin{tabular}{lccccc}
\toprule
\textbf{Method} & $\tau{=}0.3$ & $\tau{=}0.7$ & $\tau{=}1.0$ & $\tau{=}1.5$ & $\tau{=}2.0$ \\
\midrule
Top-$k$ & 82.29 & 79.53 & 76.09 & 65.73 & 51.98 \\
\rowcolor{gray!10}
Top-$p$ (nucleus) & 82.34 & 81.72 & 79.11 & 63.75 & 48.39 \\
Min-$p$ & 82.29 & 79.53 & 76.72 & 52.08 & 44.80 \\
\rowcolor{gray!10}
$\varepsilon$-Sampling & 81.93 & 79.48 & 76.88 & 67.14 & 54.84 \\
$\eta$-Sampling & 82.29 & 79.53 & 76.82 & 65.94 & 48.80 \\
\midrule
Greedy& \multicolumn{5}{c}{83.12} \\
\bottomrule
\end{tabular}
\end{table*}

\begin{table*}[htbp]
\centering
\caption{Accuracy (in \%) of different decoding/sampling strategies on ChartQA benchmark (seed=123) using Qwen2.5-VL (3B).}
\label{tab:chartqa-seed123}
\small
\setlength{\tabcolsep}{6pt}
\begin{tabular}{lccccc}
\toprule
\textbf{Method} & $\tau{=}0.3$ & $\tau{=}0.7$ & $\tau{=}1.0$ & $\tau{=}1.5$ & $\tau{=}2.0$ \\
\midrule
Top-$k$ & 82.50 & 79.38 & 75.78 & 67.14 & 51.41 \\
\rowcolor{gray!10}
Top-$p$ (nucleus) & 82.40 & 80.68 & 78.33 & 62.60 & 46.88 \\
Min-$p$ & 82.50 & 79.43 & 76.20 & 53.02 & 42.39 \\
\rowcolor{gray!10}
$\varepsilon$-Sampling & 82.50 & 79.84 & 75.99 & 66.20 & 51.35 \\
$\eta$-Sampling & 82.50 & 79.58 & 75.89 & 63.91 & 48.39 \\
\midrule
Greedy& \multicolumn{5}{c}{83.12} \\
\bottomrule
\end{tabular}
\end{table*}

\begin{table*}[htbp]
\centering
\caption{Accuracy (in \%) of different decoding/sampling strategies on ChartQA benchmark (seed=456) using Qwen2.5-VL (3B).}
\label{tab:chartqa-seed456}
\small
\setlength{\tabcolsep}{6pt}
\begin{tabular}{lccccc}
\toprule
\textbf{Method} & $\tau{=}0.3$ & $\tau{=}0.7$ & $\tau{=}1.0$ & $\tau{=}1.5$ & $\tau{=}2.0$ \\
\midrule
Top-$k$ & 82.60 & 79.38 & 76.41 & 66.88 & 52.40 \\
\rowcolor{gray!10}
Top-$p$ (nucleus) & 82.08 & 81.04 & 79.43 & 62.76 & 46.09 \\
Min-$p$ & 82.60 & 79.38 & 75.57 & 51.98 & 47.97 \\
\rowcolor{gray!10}
$\varepsilon$-Sampling & 82.60 & 80.16 & 76.20 & 67.81 & 52.19 \\
$\eta$-Sampling & 82.60 & 79.38 & 76.77 & 64.95 & 47.97 \\
\midrule
Greedy& \multicolumn{5}{c}{83.12} \\
\bottomrule
\end{tabular}
\end{table*}

\begin{table*}[htbp]
\centering
\caption{Accuracy (in \%) of different decoding/sampling strategies on ChartQA benchmark (seed=789) using Qwen2.5-VL (3B).}
\label{tab:chartqa-seed789}
\small
\setlength{\tabcolsep}{6pt}
\begin{tabular}{lccccc}
\toprule
\textbf{Method} & $\tau{=}0.3$ & $\tau{=}0.7$ & $\tau{=}1.0$ & $\tau{=}1.5$ & $\tau{=}2.0$ \\
\midrule
Top-$k$ & 82.24 & 80.16 & 76.15 & 65.42 & 51.77 \\
\rowcolor{gray!10}
Top-$p$ (nucleus) & 82.97 & 81.09 & 78.85 & 64.17 & 49.48 \\
Min-$p$ & 82.24 & 80.16 & 75.89 & 52.19 & 47.50 \\
\rowcolor{gray!10}
$\varepsilon$-Sampling & 82.24 & 80.16 & 77.03 & 66.46 & 51.46 \\
$\eta$-Sampling & 82.24 & 80.16 & 76.35 & 62.45 & 47.50 \\
\midrule
Greedy& \multicolumn{5}{c}{83.12} \\
\bottomrule
\end{tabular}
\end{table*}

\begin{table*}[htbp]
\centering
\caption{Accuracy (in \%) of different decoding/sampling strategies on BLINK benchmark (seed=42) using Qwen2.5-VL (3B).}
\label{tab:blink-seed42}
\small
\setlength{\tabcolsep}{6pt}
\begin{tabular}{lccccc}
\toprule
\textbf{Method} & $\tau{=}0.3$ & $\tau{=}0.7$ & $\tau{=}1.0$ & $\tau{=}1.5$ & $\tau{=}2.0$ \\
\midrule
Top-$k$ & 34.68 & 32.75 & 32.26 & 29.33 & 27.13 \\
\rowcolor{gray!10}
Top-$p$ (nucleus) & 34.24 & 29.71 & 31.44 & 28.60 & 26.60 \\
Min-$p$ & 35.32 & 33.07 & 30.76 & 29.97 & 28.63 \\
\rowcolor{gray!10}
$\varepsilon$-Sampling & 34.28 & 32.27 & 30.62 & 30.74 & 26.59 \\
$\eta$-Sampling & 33.13 & 32.07 & 31.16 & 27.63 & 26.10 \\
\midrule
Greedy& \multicolumn{5}{c}{35.32} \\
\bottomrule
\end{tabular}
\end{table*}

\begin{table*}[htbp]
\centering
\caption{Accuracy (in \%) of different decoding/sampling strategies on BLINK benchmark (seed=123) using Qwen2.5-VL (3B).}
\label{tab:blink-seed123}
\small
\setlength{\tabcolsep}{6pt}
\begin{tabular}{lccccc}
\toprule
\textbf{Method} & $\tau{=}0.3$ & $\tau{=}0.7$ & $\tau{=}1.0$ & $\tau{=}1.5$ & $\tau{=}2.0$ \\
\midrule
Top-$k$ & 34.68 & 33.07 & 30.40 & 29.00 & 25.71 \\
\rowcolor{gray!10}
Top-$p$ (nucleus) & 34.61 & 32.37 & 30.82 & 29.72 & 27.81 \\
Min-$p$ & 35.32 & 33.40 & 30.41 & 30.76 & 31.37 \\
\rowcolor{gray!10}
$\varepsilon$-Sampling & 34.19 & 31.55 & 30.66 & 29.49 & 25.99 \\
$\eta$-Sampling & 34.95 & 30.10 & 30.70 & 27.27 & 24.25 \\
\midrule
Greedy& \multicolumn{5}{c}{35.32} \\
\bottomrule
\end{tabular}
\end{table*}

\begin{table*}[htbp]
\centering
\caption{Accuracy (in \%) of different decoding/sampling strategies on BLINK benchmark (seed=456) using Qwen2.5-VL (3B).}
\label{tab:blink-seed456}
\small
\setlength{\tabcolsep}{6pt}
\begin{tabular}{lccccc}
\toprule
\textbf{Method} & $\tau{=}0.3$ & $\tau{=}0.7$ & $\tau{=}1.0$ & $\tau{=}1.5$ & $\tau{=}2.0$ \\
\midrule
Top-$k$ & 34.38 & 31.01 & 30.31 & 29.72 & 26.17 \\
\rowcolor{gray!10}
Top-$p$ (nucleus) & 34.29 & 32.42 & 32.58 & 28.34 & 26.76 \\
Min-$p$ & 35.32 & 32.94 & 32.08 & 29.55 & 30.72 \\
\rowcolor{gray!10}
$\varepsilon$-Sampling & 34.04 & 30.67 & 31.88 & 28.43 & 26.62 \\
$\eta$-Sampling & 34.98 & 30.58 & 31.34 & 28.74 & 26.37 \\
\midrule
Greedy& \multicolumn{5}{c}{35.32} \\
\bottomrule
\end{tabular}
\end{table*}

\begin{table*}[htbp]
\centering
\caption{Accuracy (in \%) of different decoding/sampling strategies on BLINK benchmark (seed=789) using Qwen2.5-VL (3B).}
\label{tab:blink-seed789}
\small
\setlength{\tabcolsep}{6pt}
\begin{tabular}{lccccc}
\toprule
\textbf{Method} & $\tau{=}0.3$ & $\tau{=}0.7$ & $\tau{=}1.0$ & $\tau{=}1.5$ & $\tau{=}2.0$ \\
\midrule
Top-$k$ & 35.32 & 33.85 & 30.31 & 28.99 & 26.39 \\
\rowcolor{gray!10}
Top-$p$ (nucleus) & 34.98 & 31.94 & 31.85 & 30.96 & 27.39 \\
Min-$p$ & 35.32 & 33.78 & 30.67 & 31.54 & 29.00 \\
\rowcolor{gray!10}
$\varepsilon$-Sampling & 34.15 & 29.87 & 31.60 & 28.22 & 25.79 \\
$\eta$-Sampling & 34.56 & 31.61 & 30.96 & 27.46 & 25.31 \\
\midrule
Greedy& \multicolumn{5}{c}{35.32} \\
\bottomrule
\end{tabular}
\end{table*}

\begin{figure}[H]
    \centering
    \includegraphics[width=\linewidth]{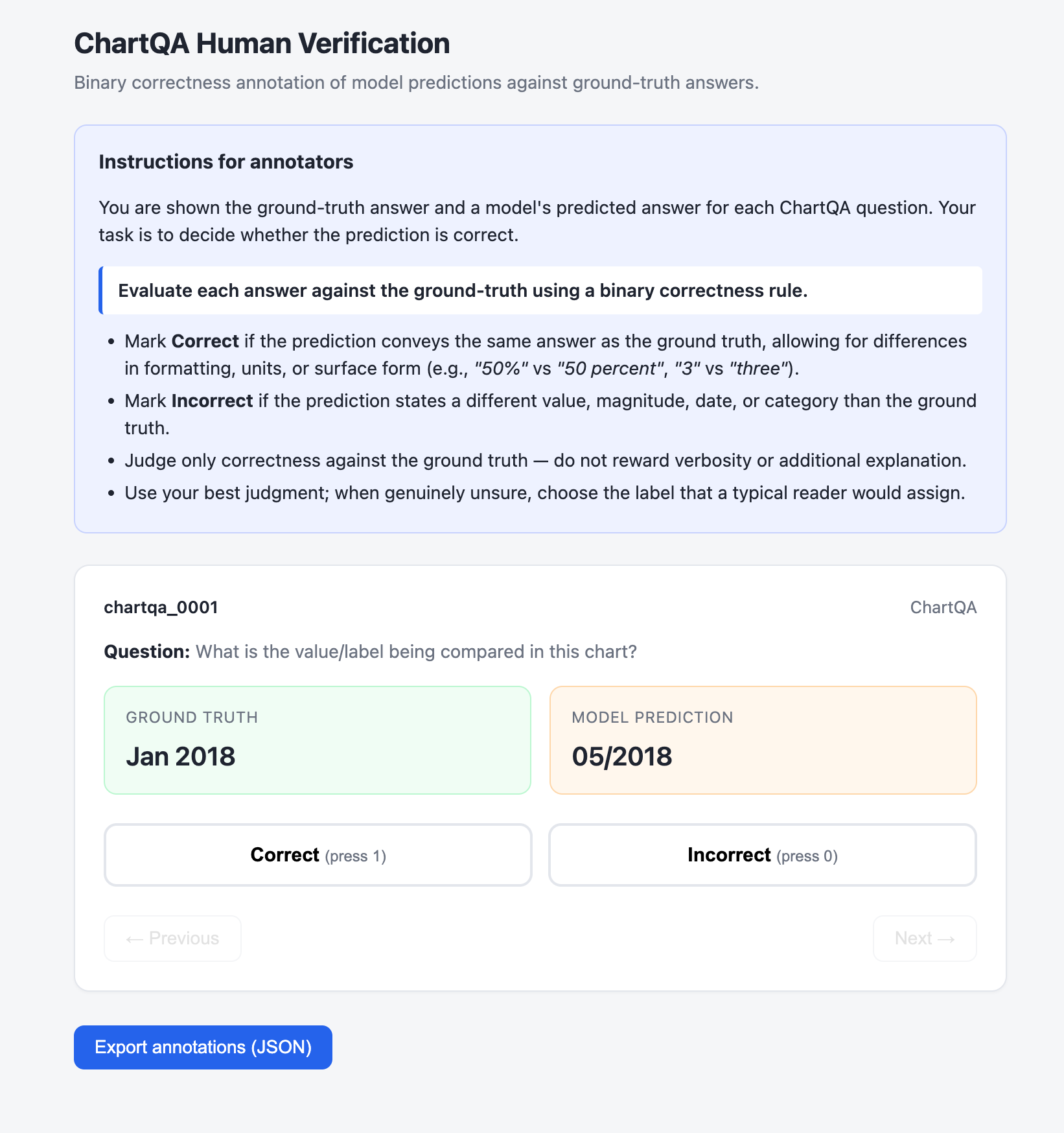}
    \caption{The annotation interface used for the human verification study on ChartQA.
    Annotators are shown the ground-truth answer and the model prediction, follow the
    binary correctness rule stated in the instructions, and record a Correct/Incorrect
    judgment.}
    \label{fig:human_annotation}
\end{figure}

\section{Human Verification Interface}
\label{sec:human_eval}
To assess whether our automatic string-matching evaluation reflects true answer correctness, we conducted a human verification study on ChartQA. For each selected question, we randomly selected one answer from the five stochastic sampling methods, so that each question contributed one sampled output without over-representing any individual decoding strategy. Two researchers independently judged each selected prediction against the ground-truth answer using a binary correctness rule (correct vs.\ incorrect), through the lightweight annotation interface shown in Figure~\ref{fig:human_annotation}. For each instance, the interface presents the ground-truth answer alongside the model prediction together with the full annotation instructions, and records a binary correctness judgment. The verbatim instruction shown to annotators is: \emph{``Evaluate each answer against the ground-truth using a binary correctness rule.''} Annotators were instructed to mark a prediction \emph{Correct} if it conveys the same answer as the ground truth up to differences in formatting, units, or surface form, and \emph{Incorrect} if it states a different value, magnitude, date, or category. Agreement between the human annotations and the string-match labels is high (Cohen's $\kappa = 0.95$), indicating that our results are not driven by string-matching artifacts.

\subsection{Annotator Recruitment and Compensation}
\label{sec:annot_recruit}
The human verification was carried out by two researchers from the project team, who
participated voluntarily. Because the task was a brief internal validation conducted by
members of the research team rather than recruited external workers, no monetary
compensation was provided or required.

\subsection{Annotator Consent}
\label{sec:annot_consent}
Prior to annotation, both researchers were shown the following instruction, which also
served as the consent statement: \emph{``You are asked to voluntarily judge
model-generated answers against the public ChartQA ground-truth labels using a binary
correctness rule. The task involves only model outputs and publicly available benchmark
data, with no personal or sensitive information. Your anonymized binary judgments will
be aggregated solely to measure the inter-annotator agreement reported in this paper. By
proceeding, you consent to this use.''} Both annotators reviewed this statement and
consented before participating.

\section{Use of AI Assistants in Research}
\label{sec:ai_assistants}
In our study, generative AI assistants are used sparingly and in accordance with the guidelines on ACL’s Policy on AI Writing Assistance. We utilize ChatGPT for basic paraphrasing and grammar checks. These tools are applied minimally to ensure the authenticity of our work and to adhere strictly to the regulatory standards set by ACL. Our use of these AI tools is focused, responsible, and aimed at supplementing rather than replacing human input and expertise in our research.

\section{Dataset Licensing and Intended Use}
\label{sec:licensing}
\subsection{Dataset License}
BLINK, ChartQA, MMMU, MMHal-Bench, MMLU, and CapArena are intended for research usage.
\paragraph{BLINK} It has an Apache license 2.0, allowing research usage.

\paragraph{ChartQA} It has a GNU General Public License v3.0, allowing research usage.

\paragraph{MMMU} It has an Apache license 2.0, allowing research usage.

\paragraph{MMHal-Bench} It has an Apache license 2.0, allowing research usage.

\paragraph{MMLU} It has an MIT license, allowing research usage.

\paragraph{CapArena} It is publicly released for academic research use~\citep{cheng2025caparena}.

\subsection{Intended Use}
Our use of the existing artifacts strictly comply with their intended purpose as benchmarks for research usage. We do not introduce any new dataset or artifact in this work.

\subsection{Privacy and Offensive Content}
\label{sec:privacy}
We use only established, publicly available benchmarks (MMMU, ChartQA, BLINK,
MM-HallBench, MMLU, and CapArena) and neither collect, scrape, nor release any new
data. These benchmarks are built from charts, diagrams, documents, and natural images
paired with factual question--answer items rather than personal records, and their
content was curated and vetted by the original dataset creators. To confirm this for
our setting, we manually inspected samples during the human verification study
(Appendix~\ref{sec:human_eval}) and the error analysis; we did not encounter personally
identifying information or offensive content. As we introduce no new data and use each
benchmark unmodified under its research license, any anonymization or content filtering
is inherited from the original releases, and no additional anonymization was required
on our part.

\section{Potential Risks}
\label{sec:risks}
This work is foundational research on inference-time decoding strategies for VQA; it
introduces no new model, dataset, or deployable system, and evaluates only publicly
released MLLMs on public benchmarks. We nonetheless note several potential risks.
First, greedy (argmax) decoding deterministically surfaces the model's most probable
answer and can therefore reflect or amplify biases present in the underlying model and
its training data; calibration-aware selection mitigates but does not eliminate this,
and we recommend pairing it with calibration checks or abstention in sensitive
applications. Second, reducing output stochasticity lowers answer diversity, which is
desirable for closed-ended VQA but may be undesirable for open-ended or creative
generation; our recommendations are scoped to closed-ended VQA and should not be
transferred uncritically to such settings. Third, although greedy decoding reduces
hallucination-like errors in our experiments, it does not remove them, and a single
deterministic answer may convey unwarranted confidence; downstream users should not
treat model outputs as authoritative without verification. Finally, all benchmarks
used are English-only and research-licensed, so our conclusions may not transfer to
other languages or domains. We are not aware of any direct path to malicious use
beyond those already inherent to the public MLLMs we evaluate.

\end{document}